\documentclass{article}

\PassOptionsToPackage{table,dvipsnames}{xcolor}
\usepackage[preprint]{neurips_2026}

\usepackage[utf8]{inputenc}
\usepackage[T1]{fontenc}
\usepackage{hyperref}
\usepackage{url}
\usepackage{booktabs}
\usepackage{amsfonts}
\usepackage{amsmath}
\usepackage{amssymb}
\usepackage{nicefrac}
\usepackage{microtype}
\usepackage{xcolor}
\usepackage{graphicx}
\usepackage{algorithm}
\usepackage{wrapfig}
\usepackage{subcaption}
\usepackage{amsthm}

\newtheorem{proposition}{Proposition}
\newtheorem{theorem}[proposition]{Theorem}

\theoremstyle{remark}
\newtheorem{remark}{Remark}

\newcommand{\E}{\mathbb{E}}

\usepackage{multirow}

\usepackage{makecell}
\usepackage[most]{tcolorbox}
\usepackage{algpseudocode}  % algorithmicx 的扩展，替代 algorithmic，更灵活

\usepackage{bm}
\usepackage{enumitem}
\usepackage{caption}

\usepackage{comment}
\usepackage{array}
\usepackage{lipsum}
\usepackage{pifont}% http://ctan.org/pkg/pifont
\usepackage{tabularx}
\usepackage{color,colortbl}
\definecolor{Gray}{gray}{0.9}

\usepackage{longtable}
\usepackage{fontawesome5}  % FontAwesome 5
\usepackage{etoc}
\usepackage{listings}
\lstdefinestyle{mypython}{
  language=Python,
  basicstyle=\ttfamily\tiny,
  keywordstyle=\color{blue!70!black}\bfseries,
  commentstyle=\color{teal!60!black}\itshape,
  stringstyle=\color{orange!70!black},
  showstringspaces=false,
  breaklines=true,
  breakatwhitespace=true,
  tabsize=4,
  numbers=left,
  numberstyle=\tiny,
  frame=single,
  framerule=0.3pt,
  rulecolor=\color{black!20},
  columns=fullflexible
}

\definecolor{promptbg}{HTML}{F4F4F5} % 匹配 SDPO 的浅灰色背景

\lstdefinestyle{promptstyle}{
    backgroundcolor=\color{promptbg},
    basicstyle=\ttfamily\small, % 等宽字体，字号较小以适应长文本
    breaklines=true,            % 自动换行
    breakatwhitespace=true,
    columns=fullflexible,
    keepspaces=true,            % 严格保留空格
    frame=none,                 % 无边框
    xleftmargin=4pt,
    xrightmargin=4pt,
    aboveskip=1em,
    belowskip=1em,
    showstringspaces=false
}
\definecolor{codegray}{rgb}{0.5,0.5,0.5}
\definecolor{codepurple}{rgb}{0.58,0,0.82}
\definecolor{backcolour}{rgb}{0.95,0.95,0.92}

\usepackage[capitalise]{cleveref}
\definecolor{color5}{HTML}{006795}
\hypersetup{
  colorlinks   = true, %Colours links instead of ugly boxes
  urlcolor     = blue, %Colour for external hyperlinks
  linkcolor    = color5, %Colour of internal links
  citecolor   = color5 %Colour of citations, could be ``red''
}
\usepackage[framemethod=TikZ]{mdframed}
\definecolor{UserExampleBg}{HTML}{ffffff}
\definecolor{UserExampleTitle}{HTML}{618197}
\newmdenv[
    roundcorner=5pt,
    backgroundcolor=UserExampleBg,
    linecolor=UserExampleTitle,
    outerlinewidth=0.5pt,
    frametitlebackgroundcolor=UserExampleTitle,
    frametitlefont={\bfseries\color{white}},
]{user_example}

\definecolor{darkblue}{rgb}{0, 0, 0.5}
\hypersetup{colorlinks=true, citecolor=darkblue, linkcolor=darkblue, urlcolor=darkblue}

\definecolor{TakeawayBlue}{HTML}{0000B3} % 深蓝色边框和标题背景
\definecolor{TakeawayBg}{HTML}{F4F5F8}   % 极淡的蓝灰色正文背景

\newtcolorbox{takeawaybox}[1]{
    enhanced,
    breakable,
    colback=TakeawayBg,
    colframe=TakeawayBlue,
    coltitle=white,
    fonttitle=\bfseries,
    title={#1},
    halign title=left,
    arc=4pt,           % 圆角大小
    boxrule=1.5pt,     % 边框粗细
    left=8pt, right=8pt, top=6pt, bottom=6pt,
    toptitle=4pt, bottomtitle=4pt,
    boxsep=2pt
}

\title{Prospective Hindsight: Self-Calibrating Reinforcement Learning via Prediction–Reality Gaps}

\author{%
  Jiaxin Zhang\thanks{ Correspondence to \texttt{\{jiaxin.zhang, wu.jason\}@salesforce.com}}, \ Xiangyu Peng, \ Qinglin Chen, \
  {Yu Li, \ Hiroaki Hayashi, \ Chien-Sheng Wu}\\[0.5em]
  Salesforce AI Research\\[0.3em]
}

\begin{document}

\maketitle

\begin{abstract}

Reinforcement learning for long-horizon agents relies on \emph{purely retrospective} training signals: credit is assigned only after observing environmental consequences, leaving the agent's belief at action time invisible to the gradient. We introduce \textbf{Prospective Hindsight (PH)}, a self-calibrating training principle that augments any retrospective base method with a signal derived from the gap between the agent's \emph{prospective prediction} (before feedback) and the \emph{retrospective evaluation} (after feedback). This per-rollout \emph{surprise} identifies samples where the agent's self-model is most inaccurate and amplifies their gradient contribution through a stop-gradient surprise-weighted advantage. Since the prospective predictor shares parameters with the policy, the two co-evolve, progressively shifting focus to the agent's remaining blind spots. We connect this principle to a privileged-information gap and show that minimizing the surprise residual provides a descent pathway on the agent's miscalibration rate; calibration thus emerges as a byproduct of optimization rather than from an added objective. On single-turn verifiable tasks and a multi-turn personal-agent task (under GRPO, on-policy distillation, and their combination), PH improves both task performance and calibration, with consistent gains across model scales. Notably, the dominant miscalibration mode shifts structurally between regimes, overconfident failures in single-turn, underconfident successes in multi-turn, yet the same training principle addresses both successfully.
\end{abstract}

\begin{figure*}[!h]
  \centering
\includegraphics[width=0.99\linewidth, clip]{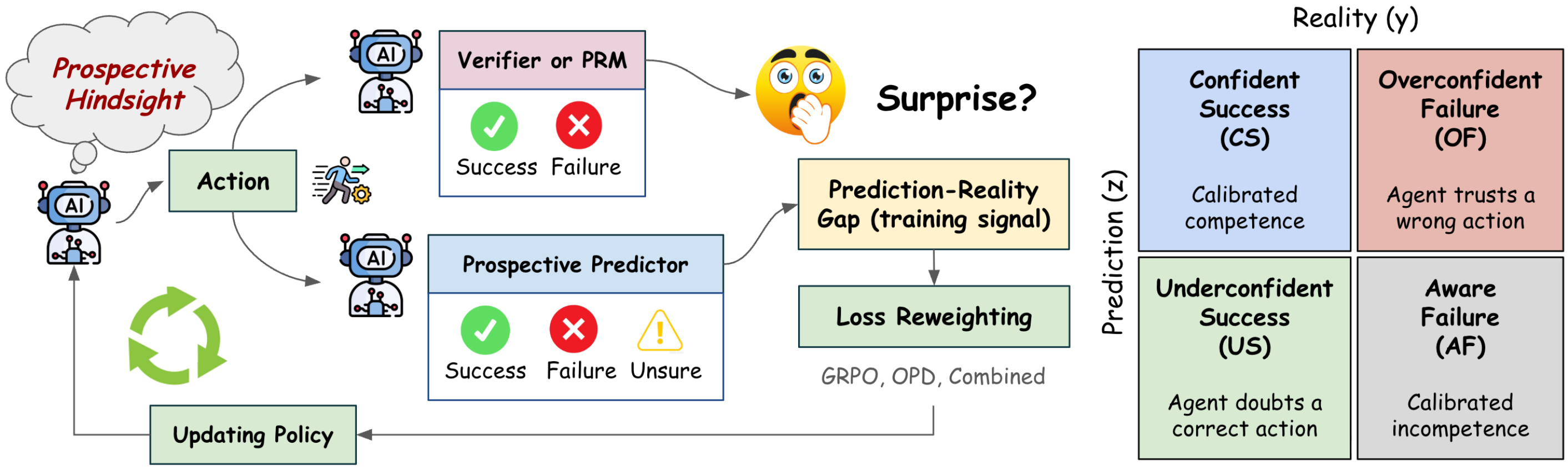}
% \vspace{-1em}
  \caption{Prospective Hindsight (PH) converts hindsight feedback into a self-calibrating training signal. PH compares an agent’s action-time prediction with the realized verifier/PRM outcome, partitions rollouts into calibrated and miscalibrated cases, and upweights surprising prediction–reality mismatches during policy updating within GRPO, on-policy distillation and their combination.}
  \label{fig:overview}
\end{figure*}

\section{Introduction}
\label{sec:intro}

LLMs increasingly act as \emph{agents}: decomposing tasks, calling tools,
writing code, conducting extended tutoring conversations, and making
intermediate commitments whose consequences unfold over many
steps~\cite{wang2026openclaw, buening2026aligning}. In these settings,
the most damaging failures are often not simple mistakes, but mistakes
made with confidence \cite{zhang2025agent-acc, zhang2025agentracer-acc,zhu2025llm-auq, zhang2026acc}. A confidently wrong agent continues down a flawed
trajectory, ignores opportunities for recovery, and triggers incorrect
downstream actions; an agent that fails while expressing uncertainty \cite{zhang2026agentic,oh2026uncertainty},
in contrast, invites a fallback, a retry, or a request for human input.
Reliable agentic learning therefore requires more than improving final
accuracy --- it requires aligning the agent's prospective belief about
its own success with the reality of its outcomes.

\textbf{The hindsight trap.}
Modern agentic RL has converged on increasingly rich \emph{retrospective}
feedback \cite{zhang2025landscape,dong2025agentic}. Process reward models score each step by inspecting the next
environmental state~\cite{wang2026openclaw, buening2026aligning}; on-policy
distillation (OPD) methods \cite{lu2025onpolicydistillation,song2026survey} construct a hindsight teacher conditioned on
post-action context unavailable to the agent at decision
time~\cite{agarwal2024policy, ye2026policy, hubotter2026reinforcement,
shenfeld2026self}. These methods differ in mechanism but share a
structural property: every learning signal is computed from quantities
observed \emph{after} the agent has acted, and the agent's prospective
belief at action time, i.e., what it expected would happen, never enters
the gradient. Two rollouts that produced the same outcome therefore
receive the same update, regardless of whether one was a confident success
and the other a lucky guess, or whether one was a failure the agent could
have anticipated and the other a failure that took the agent by surprise.
We call this structural obstruction the \emph{hindsight trap}, and it is
the reason why retrospective training reliably improves task accuracy
without resolving overconfidence~\cite{zhang2026illusion, leng2024taming,
damani2025beyond,kalai2025language,groot2024overconfidence-acc}.

\textbf{Calibration as a training signal, not a diagnostic.}
A substantial literature on LLM calibration has measured this overconfidence
and proposed inference-time fixes: verbalized confidence prompts, semantic
uncertainty estimators, post-hoc temperature
scaling~\cite{guo2017calibration, kadavath2022language-auq,
tian2023just-auq, kirchhofposition-auq}. Recent work extends these ideas
into RL by adding explicit confidence-target losses or rewriting confidence
tokens in the teacher distribution~\cite{zhang2026illusion, leng2024taming,
damani2025beyond}, but in doing so introduces a separate calibration
objective alongside the task objective; the two are reported to trade
against each other. We approach the problem from a different angle. The
gap between the agent's prospective belief and the eventual outcome is not
just a diagnostic of how well-calibrated the agent is; instead, it is a
\emph{training signal} that points to exactly the rollouts where
action-time information was insufficient.

\textbf{Why prediction--reality gaps are informative.}
The same property that makes the prospective belief invisible to
retrospective losses also makes the gap between belief and outcome
informative. By a law-of-total-variance argument, the residual variance
in the eventual outcome given action-time information decomposes into
aleatoric noise (irreducible even with privileged context) plus a
non-negative \emph{privileged-information gap}: the share of outcome
variance that no action-time predictor can resolve. This places agent RL
inside Learning Under Privileged Information (LUPI)
framework~\cite{vapnik2009new, pechyony2010theory, lopez2015unifying}.
The rollouts that materialize this gap, overconfident failures and
underconfident successes, are exactly where action-time information
falls short of post-action information. Standard hindsight-only methods
see these rollouts only through their outcome label and treat them
indistinguishably from calibrated ones; we ask what changes if training
explicitly uses the disagreement.

% We introduce \textbf{Prospective Hindsight (PH)}, a simple and general
% principle that augments any retrospective base method with a
% self-calibrating training signal derived from this gap. Before each
% rollout is scored by the verifier or PRM, we re-query the same policy with
% a different prompt to elicit a \emph{prospective prediction} of the
% eventual outcome; comparing this prediction to the retrospective
% evaluation defines a per-rollout \emph{surprise} that identifies the
% training samples where the agent's self-model is most inaccurate. PH
% amplifies the gradient contribution of these samples through a
% single-hyperparameter loss reweighting, leaving every other component of
% the base method untouched and adding zero deployment-time parameters.
% Because the predictor shares parameters with the policy, the prospective
% prediction is not an external head bolted onto the model; it co-evolves
% with the policy as training progresses, so the agent's self-model
% implicitly improves alongside the agent's behavior. This induces a
% self-calibrating loop in which the surprise signal progressively shifts
% focus to the agent's remaining blind spots and naturally vanishes once
% the prediction--reality gap closes --- without an explicit calibration
% loss being added to the optimization, and without the
% capability--calibration trade-offs reported by prior calibration-augmented
% RL recipes. Our key contributions are: 
We introduce \textbf{Prospective Hindsight (PH)}, a self-calibrating
training principle that augments any retrospective base method with a
signal derived from the prediction--reality gap. Before each rollout is
scored by the verifier or PRM, we re-query the same policy with a different
prompt to elicit a \emph{prospective prediction} of the eventual outcome;
comparing this prediction to the retrospective evaluation yields a
per-rollout \emph{surprise} that identifies the training samples where the
agent's self-model is most inaccurate. PH amplifies the gradient
contribution of these samples through a stop-gradient reweighting of the
base loss, preserving the rollout sampling distribution and the internal
structure of GRPO, OPD, or their combination. Because
the prospective predictor shares parameters with the policy rather than
being an external head bolted onto the model, the predictor co-evolves with
the policy as training progresses: the agent's self-model improves
implicitly alongside the agent's behavior, the surprise signal
progressively shifts focus to the agent's remaining blind spots, and the
reweighting diminishes as the prediction--reality gap closes. PH therefore
reduces miscalibration as a byproduct of optimization rather than through
an added calibration loss, sidestepping the capability--calibration
trade-off \cite{leng2024taming,damani2025beyond}. Our main contributions are summarized as follows:

% \vspace{-1mm}
\begin{itemize}[leftmargin=1.6em]
    \item We identify the \emph{hindsight trap} as a structural blind spot
    of retrospective RL and connect the resulting prediction--reality
    gap to a privileged-information quantity within the LUPI framework.
    \item We propose \emph{Prospective Hindsight}, a principle method that
    converts this gap into a self-calibrating training signal compatible
    with binary RL (GRPO), On-policy distillation, and their combination.
    \item We define a calibration taxonomy by pairing the prospective belief with the verifier outcome yields four
    rollout cells that decompose calibration along two
    interpretable axes during training.
    % \item Across single-turn verifiable tasks and a multi-turn
    % personal-agent setting, PH improves both task performance and
    % calibration under a single shared hyperparameter, while revealing that
    % the dominant miscalibration mode shifts structurally between regimes
    % --- overconfident failures in single-turn, underconfident successes
    % in multi-turn --- yet a single one-line rule covers both.
    \item PH improves both task performance and calibration on
    single-turn verifiable tasks  and
    on a multi-turn personal-agent task, with consistent gains across
    model scales. The dominant miscalibration
    mode shifts structurally between regimes and PH
    addresses both successfully.
\end{itemize}

\section{The Hindsight Trap in Agent RL}
\label{sec:problem}

% Current RL methods for training language model agents share a defining
% feature: every learning signal arrives \emph{after} the agent has acted. We
% call this design \emph{hindsight-only}. \S\ref{sec:agent_rl}--%
% \S\ref{sec:retro_methods} make the property precise across the base methods
% we use; \S\ref{sec:blind_spot} characterizes the resulting
% \emph{self-evaluation blind spot}; and \S\ref{sec:gap_informative} shows
% that the prediction--reality gap left by this blind spot is structurally
% informative.

\subsection{Agentic Reinforcement Learning}
\label{sec:agent_rl}

We consider an agent represented by a policy $\pi_\theta$ acting in
environments with verifiable outcomes. At each step $t \in \{1, \ldots, T\}$,
the agent observes a state $s_t$ (a query, a conversation history, or a
partial task description), produces an action $a_t \sim \pi_\theta(\cdot
\mid s_t)$, and receives a step-level outcome
\begin{equation}
    r_t \;=\; R(s_t, a_t, c_t) \;\in\; \{-1, 0, +1\}, \qquad
    y_t \;=\; \mathbf{1}[r_t > 0] \;\in\; \{0, 1\},
    \label{eq:reward}
\end{equation}
where $c_t$ denotes \emph{privileged context} that the evaluator $R$ observes
but the agent does not at action time. This covers single-step tasks
($T{=}1$, with $c_t$ a ground-truth reference) and multi-turn tasks
($T{>}1$, with $c_t{=}s_{t+1}$ the next environmental state). A trajectory
is $\tau = \{(s_t, a_t, r_t)\}_{t=1}^T$; the base methods of
\S\ref{sec:retro_methods} differ in how $r_t$ and $c_t$ shape the policy
update.

\subsection{Retrospective Training Methods}
\label{sec:retro_methods}

Current methods construct the policy update from $r_t$, optionally combined
with $c_t$. Scalar-advantage methods such as GRPO~\cite{shao2024deepseekmath}
normalize the reward within a group of rollouts sharing the same query;
token-level distillation methods such as OPD~\cite{agarwal2024policy,
ye2026policy, hubotter2026reinforcement} use $c_t$ to construct a hindsight
teacher whose action distribution is distilled into the student; combined
methods sum the two signals~\cite{wang2026openclaw}. Full advantage
formulations are given in App.~\ref{app:base_methods}.
All three families share the same hindsight-only property: the learning
signal is computed entirely from post-action quantities ($r_t$, $c_t$, or
both), and the agent's prospective belief at action time is invisible to
the update. We abstract any such instantiation as a base loss
$\ell^{\mathrm{base}}$ whose precise form does not matter for the analysis
below.

\subsection{The Self-Evaluation Blind Spot}
\label{sec:blind_spot}

Hindsight-only feedback collapses two distinct questions about each step
into a single label. The first is observable from the retrospective
evaluator: \emph{what happened}, captured by $r_t$. The second is invisible:
\emph{whether the agent could have anticipated what would happen} when it
chose $a_t$. Formally, let $z_t \in \{-1, 0, +1\}$ denote the agent's prospective belief
about $y_t$, formed from $(s_t, a_t)$ alone before $c_t$ is observed:
$z_t = +1$ encodes an expectation of success, $z_t = -1$ an expectation of
failure, and $z_t = 0$ an explicit acknowledgement of uncertainty
(\emph{honest uncertainty}). The base loss is a
function of post-action quantities alone,
$\ell^{\mathrm{base}} = \mathcal{F}(\tau, \{c_t\}_{t=1}^T)$, and therefore
independent of any such $z_t$:
\begin{equation}
    \mathbb{E}\!\left[\nabla_\theta \ell^{\mathrm{base}} \,\big|\, z_t, y_t\right]
    \;=\; \mathbb{E}\!\left[\nabla_\theta \ell^{\mathrm{base}} \,\big|\, y_t\right]
    \qquad \forall\, z_t \in \{-1, 0, +1\},
    \label{eq:blindspot}
\end{equation}
i.e., the gradient distribution depends on the realized outcome $y_t$ alone
and is invariant to the agent's prospective belief. 

We call the absence of supervision on $z_t$ the \emph{self-evaluation blind
spot}: standard updates assign the same gradient to two rollouts that
produced the same outcome, regardless of whether the agent acted with
calibrated belief, miscalibrated belief, or honest uncertainty. Two
consequences follow on the calibrated subset $z_t \in \{-1, +1\}$
(the $z_t{=}0$ case inherits the same invariance and is treated as a neutral
class). A confident success and a lucky success
receive identical reinforcement, even though the former reflects calibrated
competence while the latter is an unrecognized capability the agent cannot
deploy with confidence; symmetrically, an overconfident failure is
penalized exactly like an aware failure, even though the former is the more
dangerous mode --- the agent's own model endorses an action the verifier
rejects, so the same mistake is likely to recur. Closing this blind spot
requires a learning signal tied to the gap between $z_t$ and $y_t$, but
only if that gap is structurally meaningful rather than noise that should
be averaged out. 

\subsection{Why Prediction--Reality Gaps Are Informative}
\label{sec:gap_informative}

To characterize the structural quantity that the discrete prospective belief
$z_t$ should approximate, consider its Bayes-optimal continuous idealization,
$f^{\star}(s_t, a_t) \,{:=}\, \mathbb{E}[y_t \mid s_t, a_t]
= \Pr(y_t {=} 1 \mid s_t, a_t)$, the minimum-MSE estimator of $y_t$
available without observing $c_t$.

\begin{proposition}[Privileged-information gap]
\label{prop:gap}
For any $(s_t, a_t)$ in the support of the rollout distribution,
\begin{equation}
    \mathrm{Var}(y_t \mid s_t, a_t)
    \;=\; \underbrace{\mathbb{E}[\mathrm{Var}(y_t \mid s_t, a_t, c_t) \mid s_t, a_t]}_{\text{aleatoric noise}}
    \;+\; \underbrace{I^{\star}(s_t, a_t)}_{\text{privileged-information gap}},
    \label{eq:gap_decomp}
\end{equation}
where $I^{\star}(s_t, a_t) := \mathrm{Var}(\mathbb{E}[y_t \mid s_t, a_t, c_t] \mid s_t, a_t) \geq 0$.
\end{proposition}

The first term is the aleatoric noise irreducible even with $c_t$; the
second, $I^{\star}(s_t, a_t)$, is precisely the contribution that no
action-time predictor $f(s_t, a_t)$ can resolve. States with large
$I^{\star}$ identify exactly where the agent's epistemic horizon is most
limited and where the prediction--reality gap is informative rather than
residual noise. This places our setting in the Learning Under Privileged
Information framework~\cite{vapnik2009new, pechyony2010theory}: signals
available at training time but not at decision time fundamentally constrain
the achievable predictor (proof and full LUPI connection in
App.~\ref{app:gap_proof}). The decomposition motivates a learning rule: maintain a self-evaluator that
produces a prospective belief $z_t$ from action-time information alone, and
reweight $\ell^{\mathrm{base}}$ as a function of the gap between $z_t$ and
$y_t$.

\section{The Prospective Hindsight Method}
\label{sec:method}

% \S\ref{sec:blind_spot} identified a self-evaluation blind spot in
% hindsight-only training, and \S\ref{sec:gap_informative} showed that the
% prediction--reality gap left by this blind spot is structurally informative.

Prospective Hindsight (PH) converts this gap into a usable training signal
through three components: a self-evaluator that produces a per-rollout belief
from action-time information; a four-cell
taxonomy that partitions rollouts by belief and outcome;
and a one-line reweighting of the base loss driven by a binary surprise
indicator. We show that
the same rule has a complementary interpretation as implicit calibration. The end-to-end procedure is summarized in Algorithm~\ref{alg:ph}(App.~\ref{app:algorithm}).

\subsection{Prospective Self-Evaluator}
\label{sec:self_evaluator}

The agent $\pi_\theta$ produces both the action $a_t$ and a prospective
self-evaluation of the same action. The self-evaluator distribution $p^{\mathrm{eval}}_\theta(\cdot
\mid s_t, a_t)$ is realized by prompting $\pi_\theta$ with a template that
asks whether the eventual outcome of $a_t$ will be a success; the answer is
parsed from a \texttt{\textbackslash boxed\{$\cdot$\}} field as a ternary
score in $\{-1, 0, +1\}$ (success / failure / honest uncertainty;
App.~\ref{app:hu}). Crucially, the self-evaluator sees only $(s_t, a_t)$ and
does not observe the privileged context $c_t$ (full templates in
App.~\ref{app:prompts}). To reduce variance, we draw $M$ independent
samples per step and aggregate by majority vote:
$
    z_t \;=\; \mathrm{maj}\!\left(z_t^{(1)}, \ldots, z_t^{(M)}\right),
     z_t^{(m)} \stackrel{\text{i.i.d.}}{\sim} p^{\mathrm{eval}}_\theta(\cdot \mid s_t, a_t).
    \label{eq:eval}
$
The only training-time overhead is $M$ extra forward passes per rollout step.
% parameter-sharing means PH adds no deployment-time parameters. 
% Tie-breaking and parsing-failure handling follow App.~\ref{app:hu}.

\subsection{A Calibration Taxonomy of Rollouts}
\label{sec:taxonomy}

As shown in Table \ref{tab:cells}, the pair $(z_t, y_t)$ partitions every rollout with $z_t \in \{-1, +1\}$
into one of four \emph{calibration cells} along two epistemic axes,
the outcome axis ($y_t$, observable from the verifier) and the anticipation
axis ($z_t$, the agent's prospective self-assessment). Diagonal cells
(CS, AF) are calibrated; off-diagonal cells (OF, US) are miscalibrated and
capture the asymmetry analyzed in \S\ref{sec:blind_spot}: OF rollouts are
the most consequential failure mode (the agent's own model endorses an
action the verifier rejects), while US rollouts identify capability the
agent has but cannot yet recognize. 
\begin{table}[ht]
\centering
\small
\caption{Calibration cells of rollouts, indexed by the agent's prospective belief $z_t \in \{-1, +1\}$ and the realized outcome $y_t \in \{0, 1\}$. Rollouts with $z_t{=}0$ are deferred to
App.~\ref{app:hu}.}
\label{tab:cells}
\begin{tabular}{l c c l}
\toprule
\textbf{Cell} & $z_t$ & $y_t$ & \textbf{Epistemic interpretation} \\
\midrule
\textbf{CS} (confident success) & $+1$ & $1$ & calibrated competence \\
\textbf{OF} (overconfident failure) & $+1$ & $0$ & miscalibrated; agent trusts a wrong action \\
\textbf{US} (underconfident success) & $-1$ & $1$ & miscalibrated; agent doubts a correct action \\
\textbf{AF} (aware failure) & $-1$ & $0$ & calibrated incompetence \\
\bottomrule
\end{tabular}
\end{table}

\subsection{The Prospective Hindsight Update Rule}
\label{sec:update_rule}

PH attaches to each rollout a binary \emph{surprise} scalar and amplifies the
base loss in proportion. 
% We deliberately keep the rule as simple as possible:a binary indicator, a single hyperparameter, and no other change to the base method.

\textbf{Surprise.}
For $z_t \in \{-1, +1\}$ and $y_t \in \{0, 1\}$, the per-rollout surprise is
the binary indicator of a prediction--reality gap,
\begin{equation}
    \xi_t \;=\; \mathbb{1}\!\bigl[(z_t, y_t) \in \{\mathrm{OF}, \mathrm{US}\}\bigr] \;\in\; \{0, 1\}.
    \label{eq:surprise}
\end{equation}
$\xi_t{=}1$ exactly when the agent's prospective belief disagrees with the
realized outcome (overconfident or underconfident); $\xi_t{=}0$ when they
agree (CS or AF). This is the per-rollout realization of the
privileged-information gap $I^\star$ from
Proposition~\ref{prop:gap}: rollouts with $\xi_t{=}1$ are exactly those
whose existence makes $I^\star$ visible from inside the agent's own model.

\textbf{Surprise-weighted update.}
Let $\ell^{\mathrm{base}}_t(\theta)$ be the per-rollout base loss
(\S\ref{sec:retro_methods}); PH is agnostic to whether the base method is
GRPO, OPD, or their combination. The PH-modified loss is
\begin{equation}
    \ell^{\mathrm{PH}}_t(\theta) \;=\; \bigl( 1 + \alpha \cdot \xi_t \bigr) \cdot \ell^{\mathrm{base}}_t(\theta),
    \qquad
    \nabla_\theta \ell^{\mathrm{PH}}_t(\theta) \;=\; \bigl( 1 + \alpha \cdot \xi_t \bigr) \cdot \nabla_\theta \ell^{\mathrm{base}}_t(\theta),
    \label{eq:ph_loss}
\end{equation}
with $\alpha \geq 0$ the only hyperparameter and $\xi_t$ a stop-gradient
constant. Calibrated rollouts ($\xi_t{=}0$) keep the base loss unchanged;
surprising rollouts ($\xi_t{=}1$) are upweighted by $1{+}\alpha$. Setting
$\alpha{=}0$ recovers the base method exactly; we use $\alpha{=}1$ as the
default. 

% Equation~\ref{eq:ph_loss} preserves every other component of the
% base method: the rollout sampling distribution, the group reward
% normalization in GRPO, the per-token aggregation in OPD, and any KL or
% entropy regularizer are all left untouched. In multi-turn tasks,
% $(z_t, y_t)$ and $\xi_t$ are recomputed per turn and the factor amplifies
% the per-turn base loss; the surprise scalar is broadcast to all action
% tokens within a rollout (single-turn) or turn (multi-turn) so that the
% per-token gradient is scaled identically.

\subsection{PH as Calibration through Reweighting}
\label{sec:calibration_view}

% The update rule above was derived diagnostically: PH amplifies the base
% gradient on rollouts whose outcomes the agent could not anticipate. 

We also develop a complementary interpretation of the same rule as an
implicit calibration mechanism. As the prospective evaluator
$p^{\mathrm{eval}}_\theta$ shares parameters with the policy $\pi_\theta$,
every PH update changes both the action distribution and the agent's
prospective belief; PH biases this joint update toward states in which
the two agree. We frame the analysis below as \emph{motivation} for why
PH would be expected to reduce miscalibration as a side effect of
optimization, rather than as a formal guarantee.
% that it does so;
% empirical evidence for the predicted alignment is in \S\ref{sec:experiments}.

\textbf{Surprise-weighted decomposition.}
The PH expected loss decomposes additively into the base objective plus
a non-negative surprise residual:
\begin{equation}
    \mathbb{E}_{\rho_\theta}\!\bigl[\ell^{\mathrm{PH}}_t\bigr]
    \;=\; \mathbb{E}_{\rho_\theta}\!\bigl[\ell^{\mathrm{base}}_t\bigr]
    \;+\; \alpha \cdot \mathcal{R}(\theta),
    \quad
    \mathcal{R}(\theta) \;:=\; \mathbb{E}_{\rho_\theta}\!\bigl[\xi_t \cdot \ell^{\mathrm{base}}_t\bigr] \;\geq\; 0,
    \label{eq:ph_decomp}
\end{equation}
where $\rho_\theta$ is the joint rollout distribution induced by
$\pi_\theta$, $p^{\mathrm{eval}}_\theta$, and the verifier
(App.~\ref{app:cal:setup}). The surprise residual $\mathcal{R}(\theta)$
is the base loss restricted to the miscalibrated cells OF and US: PH
minimizes a regularized objective in which surprising rollouts contribute
through both their own loss and the surprise penalty.

\textbf{Miscalibration rate.}
Define the agent's miscalibration rate against the verifier:
\begin{equation}
    \mathcal{M}(\theta) \;:=\; \mathbb{E}_{\rho_\theta}\!\bigl[\xi_t\bigr]
    \;=\; \Pr\nolimits_{\rho_\theta}\!\bigl[(z_t, y_t) \in \{\mathrm{OF}, \mathrm{US}\}\bigr],
    \label{eq:miscal_main}
\end{equation}
which equals the binary Brier score of the deterministic predictor
$z'_t {=} (z_t{+}1)/2$ against $y_t$ (App.~\ref{app:cal:bound}). Let
$
    c(\theta) \;:=\; \mathbb{E}_{\rho_\theta}\!\bigl[\ell^{\mathrm{base}}_t \,\big|\, \xi_t = 1\bigr]
$
denote the average base loss on surprising rollouts. By the tower rule,
\begin{equation}
       \mathcal{R}(\theta) \;=\; c(\theta) \cdot \mathcal{M}(\theta),
    \label{eq:exact_residual} 
\end{equation}
which is an \emph{exact} identity rather than an inequality
(Prop.~\ref{prop:residual_miscal} of App.~\ref{app:cal:bound}). Eq.~\ref{eq:exact_residual}
admits two descent pathways for the residual: (i) reducing
$\mathcal{M}(\theta)$ by aligning the prospective belief with the verifier
outcome, or (ii) reducing $c(\theta)$ by lowering the base loss on the
miscalibrated set. Each PH gradient step is one or the other (or both),
without an explicit calibration loss ever entering the optimization.

\textbf{When does the residual identity bind?}
Eq.~\ref{eq:exact_residual} is informative as a calibration lever in
the regime where $c(\theta)$ remains bounded away from zero. Three
training phases follow naturally. Early on, both $c(\theta)$ and
$\mathcal{M}(\theta)$ are large, and minimizing $\mathcal{R}(\theta)$
admits both pathways simultaneously; this is the regime where PH
delivers its largest empirical gains (\S\ref{sec:experiments}). In
mid-training, both factors decrease as the agent's self-model improves.
At convergence, $c(\theta) \to 0$ and $\mathcal{M}(\theta) \to 0$
together, so $\mathcal{R}(\theta)$ vanishes and PH and the base method
share the same fixed points (App.~\ref{app:cal:discussion}). The
calibration interpretation is therefore not a uniform statement about
the optimization landscape but a description of how PH's residual
\emph{decomposes} during the operative training window. 

\textbf{Reported metric.}
Among the two cells contributing to $\mathcal{M}(\theta)$, OF carries
the dominant practical weight (\S\ref{sec:blind_spot}). We therefore
report the \emph{overconfident-failure rate}
\begin{equation}
    \mathrm{OFR}(\theta)
    \;:=\; \Pr\nolimits_{\rho_\theta}\!\bigl[\,y_t {=} 0 \,\big|\, z_t {=} +1\,\bigr]
    \;=\; \tfrac{\Pr[\mathrm{OF}]}{\Pr[z_t = +1]},
    \label{eq:ofr_main}
\end{equation}
the OF-conditional projection of $\mathcal{M}(\theta)$, alongside task
success. This connects PH to implicit calibration of language models~\cite{kadavath2022language-auq,
tian2023just-auq, lin2022teaching-acc}, in which post-training feedback
induces well-calibrated confidence without an explicit calibration loss;
PH makes such implicit calibration the explicit driver of optimization
rather than an emergent side effect. We emphasize that this connection
is \emph{diagnostic}: PH does not minimize $\mathcal{M}(\theta)$ as an
explicit objective, and the alignment between the PH update direction
and $-\nabla_\theta \mathcal{M}(\theta)$ is a quantitative motivation
rather than a pointwise guarantee (App.~\ref{app:cal:gradient}). 

\begin{table*}[t]
\centering
% \small
\footnotesize
\caption{%
  {Single-turn results.}
  Validation accuracy is the average / maximum logged \texttt{mean@16}
  (\%); OFR ${=}\mathrm{OF}/(\mathrm{CS}{+}\mathrm{OF})$
  (failure rate among predicted successes) and Surprise
  ${=}\mathrm{OF}{+}\mathrm{US}$ (total prediction--reality mismatch percentage);\emph{Early}/\emph{Late} are the first / last 5 training steps and
  $\Delta{=}\text{Late}-\text{Early}$.
}
\label{tab:single_results}
\begin{tabular}{ll|cccccccc}
\toprule
Dataset & Method & \multicolumn{2}{c}{Val. Acc.} & \multicolumn{3}{c}{OFR} & \multicolumn{3}{c}{Surprise} \\
\cmidrule(lr){3-4}\cmidrule(lr){5-7}\cmidrule(lr){8-10}
 &  & Avg. $\uparrow$ & Max $\uparrow$ & Early & Late $\downarrow$ & $\Delta$ $\downarrow$ & Early & Late $\downarrow$ & $\Delta$ $\downarrow$ \\
\midrule
\multirow{6}{*}{Science Q\&A} & SDPO & 63.0 & 73.4 & 75.3 & 34.2 & -41.1 & 70.9 & 35.2 & -35.7 \\  
 & +Random   & 65.4 & 74.3 & 74.1 & 33.3 & -41.7 & 69.0 & 35.3 & -33.7 \\
 & +Failure-only & 65.1 & 74.7 & 73.9 & 25.0 & -48.9 & 69.0 & 24.4 & -44.6 \\ \cmidrule(l){2-10}
  \rowcolor{blue!10} & +PH ($\alpha=0.5$) & 65.7 & 75.5 & 76.8 & \textbf{20.9} & \textbf{-56.0} & 71.7 & {21.8} & \textbf{-49.9} \\
  \rowcolor{blue!10} & +PH ($\alpha=1.0$) & 64.0 & 74.2 & 73.2 & 23.9 & -49.3 & 68.6 & \textbf{21.1} & -47.5 \\
 & +PH ($\alpha=2.0$) & \textbf{65.8} & \textbf{76.3} & 74.4 & 24.9 & -49.5 & 69.9 & 24.8 & -45.1 \\
\midrule
\multirow{6}{*}{Tool Use} & SDPO & {57.7} & {60.2} & 77.7 & 40.7 & -37.0 & 72.8 & 41.6 & -31.2 \\
 & +Random  & 58.2 & 61.6 & 77.4 & 36.3 & -41.0 & 72.3 & 36.2 & -36.1 \\
 & +Failure-only & 58.1 & 60.6 & 76.9 & 39.3 & -37.6 & 72.7 & 39.2 & -33.5 \\ \cmidrule(l){2-10}
 \rowcolor{blue!10} & +PH ($\alpha=0.5$) & 57.3 & 61.0 & 77.9 & 36.2 & -41.7 & 73.0 & 36.2 & -36.8 \\
 \rowcolor{blue!10} & +PH ($\alpha=1.0$) & \textbf{58.2} & \textbf{61.7} & 76.6 & \textbf{33.6} & \textbf{-43.0} & 71.9 & \textbf{33.5} & \textbf{-38.4} \\
 & +PH ($\alpha=2.0$) & \textbf{58.2} & 60.9 & 77.6 & {38.2} & {-39.4} & 73.0 & {38.3} & {-34.7} \\
\bottomrule
\end{tabular}
\end{table*}

\vspace{-1em}
\begin{figure*}[!h]
  \centering
\includegraphics[width=0.55\linewidth, clip]{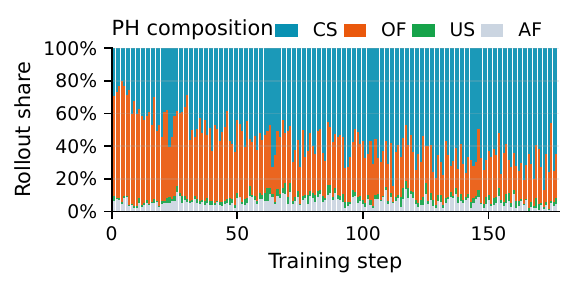}
\includegraphics[width=0.42\linewidth, clip]{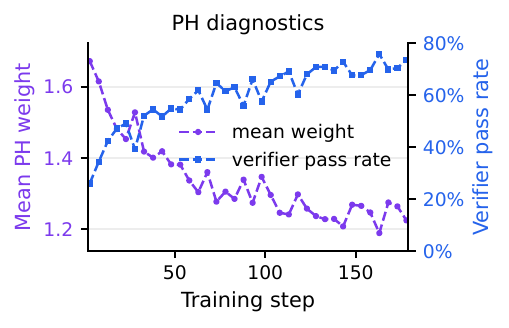}
\vspace{-1em}
  \caption{%
    \textbf{Self-evaluation dynamics under PH on Science Q\&A ($\alpha{=}1$).}
    Left: per-step composition of the four PH cells (CS / OF / US / AF);
    overconfident failures shrink while confident successes grow.
    Right: the mean PH weight decays as the verifier pass rate rises,
    indicating that PH acts as an adaptive training signal that anneals
    as the prediction--reality gap closes.
  }
  \label{fig:single_dynamics}
\end{figure*}

\vspace{-1em}
%==============================================================================
\section{Experiments}
\label{sec:experiments}
%==============================================================================

\subsection{Experimental Setup}
\label{sec:exp_setup}

We evaluate PH in two complementary regimes that stress retrospective signals of different fidelity. \textbf{Single-turn} uses verifiers on
\textsc{Science Q\&A}~\cite{feng2024sciknoweval} and \textsc{Tool Use}%
~\cite{tang2023toolalpaca} with SDPO~\cite{hubotter2026reinforcement} as the
base loss. \textbf{Multi-turn}
uses an LLM-based PRM judge on the OpenClaw-RL~\cite{wang2026openclaw}
\textsc{Personal Agent} task (GPT-4.1 user simulator), where PH
plugs into RL (GRPO~\cite{shao2024deepseekmath}), OPD~\cite{agarwal2024policy,ye2026policy,hubotter2026reinforcement}, and their Combined variant. All main
runs use \textsc{OLMo-3-7B-Instruct}~\cite{olmo2025olmo} as both policy and self-evaluator (distinct prompts, shared parameters). The prospective predictor
uses $M{=}1$ sample in single-turn (deterministic verifier) and $M{=}3$ in
multi-turn (stochastic PRM); cross-scale evidence on \textsc{gpt-oss-20B}%
~\cite{agarwal2025gptoss} and full hyperparameters are deferred to
App.~\ref{app:experimental_setup}.

\textbf{Baselines and ablations.}
The boundary case $\alpha{=}0$ recovers the base method exactly
(\emph{Baseline}). To separate PH's surprise-based \emph{selection} of which
rollouts to upweight from the contribution of additional gradient mass, we
compare against two strong reweighting controls (single-turn,
App.~\ref{app:baseline_protocols}): \textbf{Random} Reweight, which assigns the same
per-step total weight as PH but to a uniformly sampled subset of rollouts, and
\textbf{Failure-only} Reweight, which upweights all verifier failures ($y_t{=}0$). Both
use the same multiplier $(1{+}\alpha)$ as PH, so only the selection rule
differs.

\textbf{Metrics and Evaluations.}
Capability: \texttt{mean@16} (single-turn) or external-judge score (multi-turn).
Calibration: $\mathrm{OFR}{=}\Pr[y_t{=}0\mid z_t{=}{+}1]$, surprise rate
$\mathcal{M}(\theta){=}\Pr[\xi_t{=}1]$, and prospective accuracy
$\mathrm{PAcc}{=}1{-}\mathcal{M}(\theta)$. Single-turn tables report
\emph{Early}/\emph{Late} as the average over the first/last 5 training steps;
multi-turn tables report \emph{Avg} over post-initial checkpoints, \emph{Best},
and $\Delta_{\mathrm{best}}{=}\text{Best}{-}\text{initial}$. All headline numbers
are aggregated over 3 seeds.

\begin{table}[t]
\centering
% \small
\footnotesize
\caption{{Multi-turn personal-agent results (GPT-4.1 simulator).}
  PH plugged into RL, OPD, and the Combined recipe; performance is the
  teacher-side personalization-evaluator score, calibration is computed
  over teacher-side rollouts within mixed training, $\Delta_{\mathrm{best}}{=}\text{Best}-\text{initial}$.
}
\label{tab:multi_results}
\begin{tabular}{ll|ccc ccc}
\toprule
& & \multicolumn{3}{c}{Capability} & \multicolumn{3}{c}{Calibration} \\
\cmidrule(lr){3-5} \cmidrule(lr){6-8}
Method & Setting & Avg. $\uparrow$ & Best $\uparrow$ & $\Delta_{\mathrm{best}}$ $\uparrow$ & OFR $\downarrow$ & Surprise $\downarrow$ & Prosp. Acc. $\uparrow$ \\
\midrule
\multirow{4}{*}{RL} & GRPO  & 68.6 & 77.7 & +8.5 & 5.3 & 16.5 & 83.5 \\
\cmidrule(l){2-8}
\rowcolor{blue!10} &  +PH ($\alpha=0.5$) & 72.6 & 77.7 & +8.7 & 4.9 & 15.0 & 85.0 \\
\rowcolor{blue!10}  &  +PH ($\alpha=1.0$) & \textbf{75.8} & \textbf{82.9} & \textbf{+14.8} & 3.1 & \textbf{12.2} & \textbf{87.8} \\
 & +PH ($\alpha=2.0$) & 68.6 & 77.3 & +9.4 & \textbf{2.5} & 16.0 & 84.0 \\
\midrule
\multirow{4}{*}{OPD} & SDPO & 54.0 & 65.6 & +2.5 & 15.9 & 37.2 & 62.8 \\
\cmidrule(l){2-8}
\rowcolor{blue!10} & +PH ($\alpha=0.5$) & 61.8 & \textbf{77.7} & +5.6 & \textbf{1.5} & \textbf{13.7} & \textbf{86.3} \\
\rowcolor{blue!10} &  +PH ($\alpha=1.0$) & \textbf{69.3} & 75.2 & +2.7 & 2.6 & 18.1 & 81.9 \\
 & +PH ($\alpha=2.0$) & 59.2 & 77.3 & \textbf{+17.1} & 5.0 & 17.2 & 82.8 \\
\midrule
\multirow{4}{*}{Combined} & RL+OPD & {61.2} & 73.5 & +0.8 & 10.0 & 40.1 & 59.9 \\
\cmidrule(l){2-8}
\rowcolor{blue!10} &  +PH ($\alpha=0.5$) & \textbf{66.5} & 75.6 & \textbf{+8.1} & 4.3 & \textbf{28.4} & \textbf{71.6} \\
\rowcolor{blue!10}  &  +PH ($\alpha=1.0$) & {65.1} & \textbf{78.5} & \textbf{+8.1} & 6.8 & 36.0 & 64.0 \\
 & +PH ($\alpha=2.0$) & 60.3 & 64.4 & +6.0 & \textbf{2.5} & 29.5 & 70.5 \\
\bottomrule
\end{tabular}
\end{table}

\vspace{-1em}
\begin{figure*}[!h]
  \centering
\includegraphics[width=0.99\linewidth, clip]{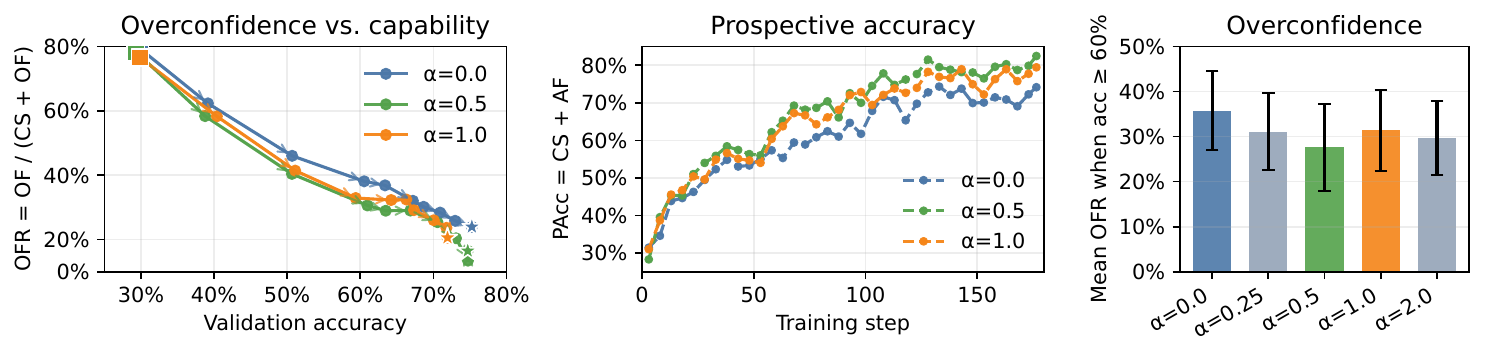}
\vspace{-1em}
\caption{%
  {Capability and calibration dynamics on Science Q\&A across $\alpha$.}
  \textbf{Left:} OFR vs.\ Validation Accuracy; \textbf{Middle:} prospective accuracy over training. \textbf{Right:} Effect of $\alpha$ on overconfidence rate.
}
\label{fig:capability_calibration}
\end{figure*}

\vspace{-1.2em}
\subsection{Single-Turn Verifiable Tasks}
\label{sec:exp_single}

% \paragraph{Headline: capability preserved, overconfidence cut by 30--40\%.}
Across both domains in Table~\ref{tab:single_results}, PH preserves or slightly
improves SDPO while substantially shrinking calibration
error. On Science Q\&A, late-training OFR drops from $34.2\%$ (SDPO) to
$\mathbf{20.9\%}$ at $\alpha{=}0.5$, and the
surprise rate tracks OFR closely ($35.2\%{\to}\mathbf{21.8}\%$). On Tool Use, $\alpha{=}1.0$ jointly attains the
best capability and calibration ($\texttt{mean@16}{=}\mathbf{58.2/61.7}$, OFR
$\mathbf{33.6\%}$, surprise $\mathbf{33.5\%}$). Figure~\ref{fig:model_scale} replicates the SDPO vs.\ SDPO+PH comparison on
\textsc{gpt-oss-20B}. Using the same $\alpha{=}1.0$, PH delivers comparable or larger improvements on all three metrics, showing that it tracks a structural property of hindsight-only training that persists across model scale.

\textbf{Selection, not mass, drives the gain.}
A natural concern is that PH simply adds gradient mass to a subset of rollouts;
any reweighting with the same per-step total weight might do equally well.
\textbf{Random} reweighting leaves
calibration essentially unchanged on Science Q\&A (OFR $33.3\%$ vs.\ SDPO's
$34.2\%$) and improves it only modestly on Tool Use; the gain over SDPO is far
smaller than PH's. \textbf{Failure-only} reweighting recovers most of PH's OFR reduction on Science Q\&A
($25.0\%$ vs.\ $20.9\%$) but trails PH on the surprise rate
($24.4\%$ vs.\ $\mathbf{21.8}\%$).  The OFR gap reflects that single-turn
miscalibration is OF-dominated (Figure~\ref{fig:cell_composition}), so a 1D outcome-axis baseline already captures most of it; the surprise-rate gap reflects the cells PH amplifies but failure-only misses underconfident successes. 

% PH's advantage is thus
% \emph{selection} along the $(z_t, y_t)$ plane, not mass alone, and its margin
% over failure-only foreshadows the much larger gap we observe in multi-turn, where US is the dominant miscalibration mode.

\textbf{PH is a self-extinguishing adaptive curriculum.}
Figure~\ref{fig:single_dynamics} shows the per-step PH-cell composition and
the corresponding PH-weight trajectory on Science Q\&A. The
overconfident-failure (OF) cell shrinks monotonically across training while
the confident-success (CS) cell grows; the underconfident-success (US) and
aware-failure (AF) cells stay small throughout. Crucially, the mean PH weight
(right) decays as the verifier pass rate rises:
PH amplifies gradient signal early when miscalibration is rampant and naturally
anneals once the policy has learned to predict its own outcomes, matching the
asymptotic argument in App.~\ref{app:cal:discussion} ($w_t{\to}1$ as
$\xi_t{\to}0$). This means PH does not require a learning-rate schedule for the
surprise weight: the data itself controls the curriculum.

\textbf{Calibration emerges as a byproduct of surprise reweighting.}
Figure~\ref{fig:capability_calibration} (middle) tracks prospective accuracy
$\mathrm{PAcc}(\theta){=}1{-}\mathcal{M}(\theta)$ across training. The pre-training self-evaluator is essentially
uncalibrated but PH lifts $\mathrm{PAcc}$ from $\sim30\%$ to $\sim80\%$ while the SDPO baseline plateaus near $73\%$. The predictor receives no direct gradient (the reweighting is stop-gradient on $\xi_t$); the gap is therefore a
\emph{byproduct} of policy optimization, evidence that PH co-trains a calibrated
self-model alongside the policy at no additional parameter cost. This is the
mechanism behind the simultaneous capability and calibration improvement:
the same gradient steps that reduce $\ell^{\mathrm{base}}$ on miscalibrated
rollouts also drive $\mathcal{M}(\theta){\to}0$ via the parameter-sharing
channel. The full $\alpha$-sweep  confirms the entire window improves both axes simultaneously, with $\alpha{=}0.5$ attaining the lowest mean OFR at matched capability.

\vspace{-1em}
\begin{figure*}[!h]
  \centering
\includegraphics[width=0.99\linewidth, clip]{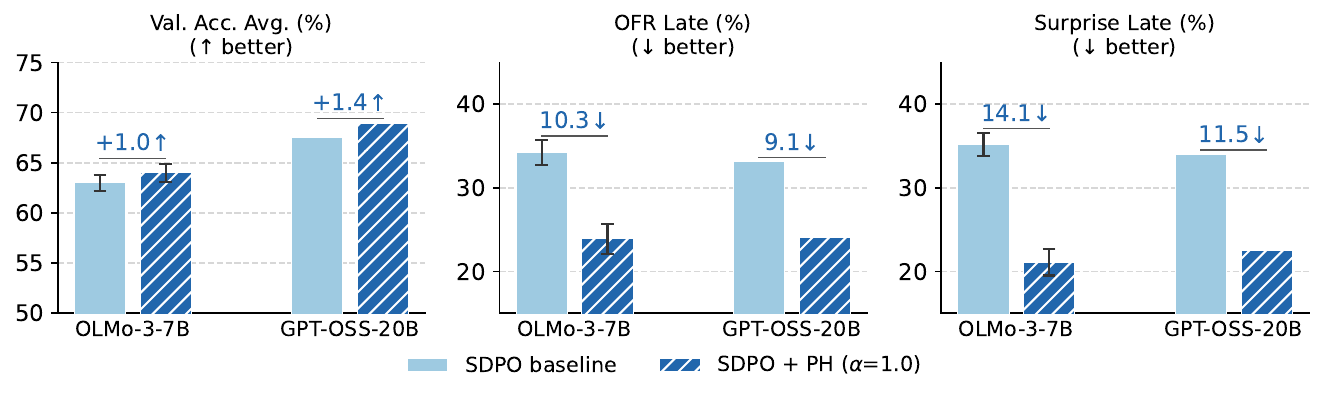}
\vspace{-1em}
  \caption{%
Capability and calibration gains from PH transfer across model scale using SDPO baseline and SDPO+PH ($\alpha{=}1.0$) on OLMo-3-7B-Instruct and gpt-oss-20B.}
  \label{fig:model_scale}
\end{figure*}

\vspace{-1.5em}
\subsection{Multi-Turn Personal Agent}
\label{sec:exp_multi}

% We move from the deterministic single-turn verifier to the multi-turn
% personal-agent setting (App.~\ref{app:multi_turn_setup}), where the
% retrospective signal is no longer a rule-based verifier but an LLM-based
% PRM judge. Here PH must (i) plug into three different base methods --- RL
% (GRPO), OPD, and the Combined recipe of \cite{wang2026openclaw} --- without
% method-specific tuning, and (ii) remain effective despite per-turn judge
% noise. Table~\ref{tab:multi_results} sweeps $\alpha\in\{0.5, 1.0, 2.0\}$
% under the GPT-4.1 user simulator.

\textbf{Multi-turn shifts the dominant miscalibration mode from OF to US.}
Before discussing capability and calibration numbers, we characterize the
\emph{kind} of miscalibration the multi-turn regime produces.
Figure~\ref{fig:cell_composition} aggregates $(z_t, y_t)$ pairs across
training under $\alpha{=}1.0$ and contrasts single-turn tasks with the
three multi-turn base methods. The picture is structurally different: while
single-turn rollouts are dominated by overconfident failures (OF$=34.1\%$,
US$=2.7\%$), multi-turn rollouts shift the dominant miscalibration to
\emph{underconfident successes} (US), which grow monotonically across base
methods, RL ($13.8\%$) $<$ OPD ($27.9\%$) $<$ Combined ($34.5\%$),
while OF stays uniformly small ($4$--$5\%$). The two regimes therefore
expose complementary failure modes: single-turn agents over-trust wrong
actions, while multi-turn agents under-trust correct ones.

\begin{wrapfigure}{r}{0.5\textwidth}
    \vspace{-2em}
    \centering
\includegraphics[width=\linewidth, clip]{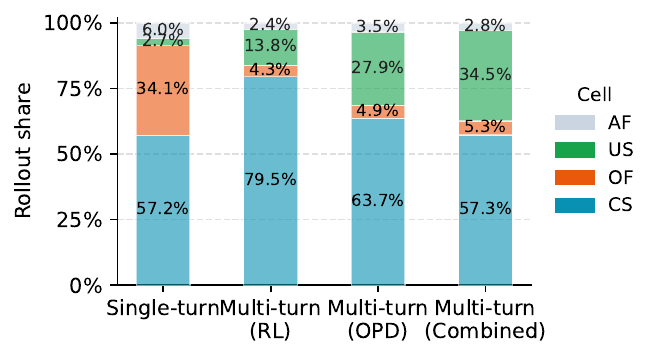}

\vspace{-1em}
\caption{PH cell composition differs structurally between single-turn and multi-turn settings.}     
\label{fig:cell_composition}
    \vspace{-1em}
\end{wrapfigure}

% Two implications
% follow. First, failure-only reweighting targets
% $y_t{=}0$ rollouts and is therefore structurally unable to address the
% dominant multi-turn miscalibration; the small Failure-only-vs-PH gap on
% single-turn is not predictive of multi-turn behavior. Second, the
% across-method monotonicity of US gives a quantitative axis along which to
% predict where PH's gain should be largest --- a prediction we now check
% against Table~\ref{tab:multi_results}.

\textbf{PH improves all three base methods, with the largest lift on OPD.}
At $\alpha{=}1.0$, PH improves the trajectory-averaged teacher-side score
over the baseline in every block of Table~\ref{tab:multi_results}:
 RL ($+7.2$), OPD
($+15.3$), and Combined
($+3.9$). The OPD block is the most striking: its baseline has the
weakest absolute performance, and
PH simultaneously stabilizes the trajectory ($+15.3$ Avg) and raises the
peak. The cell-composition picture above is consistent
with the OPD-vs-RL ordering: OPD's privileged-context teacher leaves the
student with substantially more underconfident successes than RL
(US: $27.9\%$ vs.\ $13.8\%$), and PH amplifies precisely those rollouts.
The Combined block is a partial exception to this ordering since its US
share ($34.5\%$) is the highest of the three, yet its PH gain is the
smallest. 

% We read this as evidence that PH amplifies a
% \emph{prediction--reality discrepancy}, not US in isolation: the Combined
% recipe already mixes scalar reward with the OPD signal, so a portion of
% the self-evaluation gap that PH would otherwise close has already been
% absorbed by the retrospective signal itself, leaving a smaller residual
% for the prospective signal to act on.

\textbf{Calibration improves uniformly at $\alpha{=}1.0$.}
At the headline $\alpha{=}1.0$, OFR drops in every block and prospective
accuracy rises in every block, peaking at $\mathbf{87.8\%}$ (RL+PH,
$\alpha{=}1.0$) and $\mathbf{86.3\%}$ (OPD+PH, $\alpha{=}0.5$). The
surprise gate concentrates gradient mass on exactly the cells (OF, US)
that contribute to $\mathcal{M}(\theta)$, and those cells empirically
shrink under PH training. A per-turn breakdown
(Figure~\ref{fig:multi_alpha_sweep}) shows that PH compresses
overconfidence most in late-turn positions, where the agent's self-model
is least developed. Larger $\alpha$ trades stability for further OFR
reduction: at $\alpha{=}2.0$, the RL block reaches the lowest OFR in the
table ($2.5\%$) but the predictor collapses toward a near-constant
$z_t{=}{-}1$ output (Surprise rate $16.0\%$, prospective accuracy
$84.0\%$), so the calibration improvement comes at the cost of a
degraded self-model rather than from genuinely better belief--outcome
agreement. Replacing the GPT-4.1 user simulator with
GPT-5.4 preserves the
qualitative pattern of Table~\ref{tab:multi_results}
(see Figure~\ref{fig:multi_simulator}). 

% The Combined-method gain is
% smaller under GPT-5.4, suggesting that the marginal value of a
% prospective signal shrinks once the retrospective signal is itself
% stronger, consistent with PH's design as an amplifier of the
% prediction--reality gap rather than a stand-alone calibration objective.

\textbf{What PH is doing, across regimes.}
The base method's gradient is invariant to the agent's prospective
belief; the surprise indicator $\xi_t$ is the per-rollout realization of
the privileged-information gap $I^\star$ that this invariance hides.
Three observations across both regimes are consistent with this view:
(i) calibration improves as a byproduct of optimization, so the same
gradient steps that drive the base loss also reduce $\mathcal{M}(\theta)$;
(ii) the dominant miscalibration cell flips between regimes yet the same
update rule covers both; and (iii) the surprise weight self-extinguishes
as the prediction--reality gap closes, so PH asymptotically reduces to
the base method. Together these suggest that what PH amplifies is not a
particular failure mode but the existence of a prediction--reality
discrepancy itself, which is why the same scalar $\alpha$ transfers
across tasks, base methods, and model scales without re-tuning.

\section{Related work}
\label{sec:related}
%==============================================================================

\textbf{Retrospective training signals in Agentic RL.}
Recent RL/agentic RL frameworks have steadily enriched the post-action signal
delivered to the policy \cite{zhang2025landscape, cui2026rethinking}. OpenClaw-RL~\cite{wang2026openclaw} unifies binary
RL (GRPO~\cite{shao2024deepseekmath}, PPO~\cite{schulman2017proximal}), on-policy distillation
(OPD~\cite{agarwal2024policy,ye2026policy,hubotter2026reinforcement,shenfeld2026self, sang2026policy, kim2026does}),
and their combination under a single asynchronous training stack \cite{zhao2026self, ding2026hdpo};
Buening et al.~\cite{buening2026aligning} further show that raw multi-turn user interactions are themselves sufficient supervision for self-distillation
without an explicit reward model. 
% Within this lineage, OPD is particularly
% relevant to our analysis: it conditions a hindsight teacher on
% post-action context that the student does not observe at decision
% time, and distills the teacher's action distribution into the student.
These methods differ in mechanism but share a structural property: every
learning signal is computed entirely from quantities observed
\emph{after} the agent has acted, so the agent's prospective belief at
action time never enters the gradient \cite{agarwal2024policy, ye2026policy, wang2026tcod}. Prospective
Hindsight is orthogonal to this axis of progress: it leaves the base
method's rollout, advantage, and regularization structure intact, and
instead converts the gap between an action-time prospective prediction and the eventual retrospective evaluation into a per-rollout reweighting that plugs into any of the above base methods.

\textbf{Calibration of language-model agents.}
Most prior work treats calibration as a property of the trained model
to be measured or controlled at inference time, through diagnostic
studies of confidence--accuracy
alignment~\cite{guo2017calibration,kadavath2022language-auq,geng2024survey-acc},
elicitation of verbalized confidence or sampling-based uncertainty
estimates~\cite{tian2023just-auq,lin2022teaching-acc,kuhnsemantic-auq},
trajectory-level uncertainty quantification in agentic
settings~\cite{kirchhofposition-auq,duan2025uprop-auq,zhao2025uncertainty-auq,liu2024uncertainty-acc,han2024towards-auq, zhang2026agentic, zhang2026acc},
and recent analyses linking hallucination to training procedures that
reward confident guesses over abstention~\cite{kalai2025language}.
A separate line moves calibration into RL by adding an explicit
confidence-target objective: CaOPD~\cite{zhang2026illusion} attributes
overconfidence in hindsight distillation to teacher--student information
asymmetry and proposes confidence-token replacement; Leng et
al.~\cite{leng2024taming} and Damani et al.~\cite{damani2025beyond}
report similar calibration-augmented RL recipes. These approaches
consistently report a \emph{capability--calibration trade-off} \cite{yang2026calibration}:
improved calibration comes at the cost of a regression in task accuracy.
PH proceeds differently: rather than adding a calibration term to
compete with the task objective, we use the per-rollout disagreement
between the agent's prospective belief and the verifier outcome to
\emph{reweight} which base-method gradients deserve more mass.
Calibration improves as a byproduct of the same gradient steps that
drive task success, and the trade-off observed in calibration-augmented
baselines does not arise because calibration and capability share a
single descent direction.

\section{Conclusion}
\label{sec:conclusion}

Prospective Hindsight treats the gap between an agent's prospective
belief and the verifier's retrospective outcome as a learning resource
rather than a diagnostic, amplifying the gradient on exactly the
rollouts where retrospective signals are uninformative. The same
principle should extend beyond the binary verifiers and per-turn PRMs studied here, to longer horizons such as software engineering, terminal-use, GUI, and research agents, and to base methods beyond GRPO and on-policy
distillation. We see PH as a first but critical step toward reliable agent-RL pipelines that
optimize task performance and self-calibration jointly, by construction
rather than through a separate calibration objective.

\bibliographystyle{plain}
\bibliography{colm2026_conference,auq,acc,ph}

@inproceedings{liu2024uncertainty-acc,
  title={Uncertainty Calibration for Tool-Using Language Agents},
  author={Liu, Hao and Dou, Zi-Yi and Wang, Yixin and Peng, Nanyun and Yue, Yisong},
  booktitle={Findings of the Association for Computational Linguistics: EMNLP 2024},
  pages={16781--16805},
  year={2024}
}

@inproceedings{geng2024survey-acc,
  title={A survey of confidence estimation and calibration in large language models},
  author={Geng, Jiahui and Cai, Fengyu and Wang, Yuxia and Koeppl, Heinz and Nakov, Preslav and Gurevych, Iryna},
  booktitle={Proceedings of the 2024 Conference of the North American Chapter of the Association for Computational Linguistics: Human Language Technologies (Volume 1: Long Papers)},
  pages={6577--6595},
  year={2024}
}

@article{lin2022teaching-acc,
  title={Teaching models to express their uncertainty in words},
  author={Lin, Stephanie and Hilton, Jacob and Evans, Owain},
  journal={arXiv preprint arXiv:2205.14334},
  year={2022}
}

@article{groot2024overconfidence-acc,
  title={Overconfidence is key: Verbalized uncertainty evaluation in large language and vision-language models},
  author={Groot, Tobias and Valdenegro-Toro, Matias},
  journal={arXiv preprint arXiv:2405.02917},
  year={2024}
}

@article{zhang2025agent-acc,
  title={Which agent causes task failures and when? on automated failure attribution of llm multi-agent systems},
  author={Zhang, Shaokun and Yin, Ming and Zhang, Jieyu and Liu, Jiale and Han, Zhiguang and Zhang, Jingyang and Li, Beibin and Wang, Chi and Wang, Huazheng and Chen, Yiran and others},
  journal={arXiv preprint arXiv:2505.00212},
  year={2025}
}

@article{zhang2025agentracer-acc,
  title={AgenTracer: Who Is Inducing Failure in the LLM Agentic Systems?},
  author={Zhang, Guibin and Wang, Junhao and Chen, Junjie and Zhou, Wangchunshu and Wang, Kun and Yan, Shuicheng},
  journal={arXiv preprint arXiv:2509.03312},
  year={2025}
}

@article{kadavath2022language-auq,
  title={Language models (mostly) know what they know},
  author={Kadavath, Saurav and Conerly, Tom and Askell, Amanda and Henighan, Tom and Drain, Dawn and Perez, Ethan and Schiefer, Nicholas and Hatfield-Dodds, Zac and DasSarma, Nova and Tran-Johnson, Eli and others},
  journal={arXiv preprint arXiv:2207.05221},
  year={2022}
}

@article{duan2025uprop-auq,
  title={UProp: Investigating the Uncertainty Propagation of LLMs in Multi-Step Agentic Decision-Making},
  author={Duan, Jinhao and Diffenderfer, James and Madireddy, Sandeep and Chen, Tianlong and Kailkhura, Bhavya and Xu, Kaidi},
  journal={arXiv preprint arXiv:2506.17419},
  year={2025}
}

@inproceedings{zhao2025uncertainty-auq,
  title={Uncertainty propagation on llm agent},
  author={Zhao, Qiwei and Li, Dong and Liu, Yanchi and Cheng, Wei and Sun, Yiyou and Oishi, Mika and Osaki, Takao and Matsuda, Katsushi and Yao, Huaxiu and Zhao, Chen and others},
  booktitle={Proceedings of the 63rd Annual Meeting of the Association for Computational Linguistics (Volume 1: Long Papers)},
  pages={6064--6073},
  year={2025}
}

@article{shinn2023reflexion-auq,
  title={Reflexion: Language agents with verbal reinforcement learning},
  author={Shinn, Noah and Cassano, Federico and Gopinath, Ashwin and Narasimhan, Karthik and Yao, Shunyu},
  journal={Advances in Neural Information Processing Systems},
  volume={36},
  pages={8634--8652},
  year={2023}
}

@article{madaan2023self-auq,
  title={Self-refine: Iterative refinement with self-feedback},
  author={Madaan, Aman and Tandon, Niket and Gupta, Prakhar and Hallinan, Skyler and Gao, Luyu and Wiegreffe, Sarah and Alon, Uri and Dziri, Nouha and Prabhumoye, Shrimai and Yang, Yiming and others},
  journal={Advances in Neural Information Processing Systems},
  volume={36},
  pages={46534--46594},
  year={2023}
}

@inproceedings{tian2023just-auq,
  title={Just Ask for Calibration: Strategies for Eliciting Calibrated Confidence Scores from Language Models Fine-Tuned with Human Feedback},
  author={Tian, Katherine and Mitchell, Eric and Zhou, Allan and Sharma, Archit and Rafailov, Rafael and Yao, Huaxiu and Finn, Chelsea and Manning, Christopher D},
  booktitle={Proceedings of the 2023 Conference on Empirical Methods in Natural Language Processing},
  pages={5433--5442},
  year={2023}
}

@inproceedings{kuhnsemantic-auq,
  title={Semantic Uncertainty: Linguistic Invariances for Uncertainty Estimation in Natural Language Generation},
  author={Kuhn, Lorenz and Gal, Yarin and Farquhar, Sebastian},
  booktitle={The Eleventh International Conference on Learning Representations},
  year={2023}
}

@inproceedings{kirchhofposition-auq,
  title={Position: Uncertainty Quantification Needs Reassessment for Large Language Model Agents},
  author={Kirchhof, Michael and Kasneci, Gjergji and Kasneci, Enkelejda},
  booktitle={Forty-second International Conference on Machine Learning Position Paper Track},
  year={2025}
}

@article{zhu2025llm-auq,
  title={Where llm agents fail and how they can learn from failures},
  author={Zhu, Kunlun and Liu, Zijia and Li, Bingxuan and Tian, Muxin and Yang, Yingxuan and Zhang, Jiaxun and Han, Pengrui and Xie, Qipeng and Cui, Fuyang and Zhang, Weijia and others},
  journal={arXiv preprint arXiv:2509.25370},
  year={2025}
}

@inproceedings{han2024towards-auq,
  title={Towards Uncertainty-Aware Language Agent},
  author={Han, Jiuzhou and Buntine, Wray and Shareghi, Ehsan},
  booktitle={Findings of the Association for Computational Linguistics ACL 2024},
  pages={6662--6685},
  year={2024}
}

@article{wang2026openclaw,
  title={OpenClaw-RL: Train Any Agent Simply by Talking},
  author={Wang, Yinjie and Chen, Xuyang and Jin, Xiaolong and Wang, Mengdi and Yang, Ling},
  journal={arXiv preprint arXiv:2603.10165},
  year={2026}
}

@article{buening2026aligning,
  title={Aligning language models from user interactions},
  author={Buening, Thomas Kleine and H{\"u}botter, Jonas and P{\'a}sztor, Barna and Shenfeld, Idan and Ramponi, Giorgia and Krause, Andreas},
  journal={arXiv preprint arXiv:2603.12273},
  year={2026}
}

@article{pechyony2010theory,
  title={On the theory of learnining with privileged information},
  author={Pechyony, Dmitry and Vapnik, Vladimir},
  journal={Advances in neural information processing systems},
  volume={23},
  year={2010}
}

@article{kim2026does,
  title={Why Does Self-Distillation (Sometimes) Degrade the Reasoning Capability of LLMs?},
  author={Kim, Jeonghye and Luo, Xufang and Kim, Minbeom and Lee, Sangmook and Kim, Dohyung and Jeon, Jiwon and Li, Dongsheng and Yang, Yuqing},
  journal={arXiv preprint arXiv:2603.24472},
  year={2026}
}

@article{sang2026policy,
  title={On-Policy Self-Distillation for Reasoning Compression},
  author={Sang, Hejian and Xu, Yuanda and Zhou, Zhengze and He, Ran and Wang, Zhipeng and Sun, Jiachen},
  journal={arXiv preprint arXiv:2603.05433},
  year={2026}
}

@article{ye2026policy,
  title={On-policy context distillation for language models},
  author={Ye, Tianzhu and Dong, Li and Wu, Xun and Huang, Shaohan and Wei, Furu},
  journal={arXiv preprint arXiv:2602.12275},
  year={2026}
}

@inproceedings{agarwal2024policy,
  title={On-policy distillation of language models: Learning from self-generated mistakes},
  author={Agarwal, Rishabh and Vieillard, Nino and Zhou, Yongchao and Stanczyk, Piotr and Garea, Sabela Ramos and Geist, Matthieu and Bachem, Olivier},
  booktitle={The twelfth international conference on learning representations},
  year={2024}
}

@inproceedings{kwon2023efficient,
  title={Efficient memory management for large language model serving with pagedattention},
  author={Kwon, Woosuk and Li, Zhuohan and Zhuang, Siyuan and Sheng, Ying and Zheng, Lianmin and Yu, Cody Hao and Gonzalez, Joseph and Zhang, Hao and Stoica, Ion},
  booktitle={Proceedings of the 29th symposium on operating systems principles},
  pages={611--626},
  year={2023}
}

@inproceedings{sheng2025hybridflow,
  title={Hybridflow: A flexible and efficient rlhf framework},
  author={Sheng, Guangming and Zhang, Chi and Ye, Zilingfeng and Wu, Xibin and Zhang, Wang and Zhang, Ru and Peng, Yanghua and Lin, Haibin and Wu, Chuan},
  booktitle={Proceedings of the Twentieth European Conference on Computer Systems},
  pages={1279--1297},
  year={2025}
}

@article{shenfeld2026self,
  title={Self-Distillation Enables Continual Learning},
  author={Shenfeld, Idan and Damani, Mehul and H{\"u}botter, Jonas and Agrawal, Pulkit},
  journal={arXiv preprint arXiv:2601.19897},
  year={2026}
}

@article{hubotter2026reinforcement,
  title={Reinforcement Learning via Self-Distillation},
  author={H{\"u}botter, Jonas and L{\"u}beck, Frederike and Behric, Lejs and Baumann, Anton and Bagatella, Marco and Marta, Daniel and Hakimi, Ido and Shenfeld, Idan and Buening, Thomas Kleine and Guestrin, Carlos and others},
  journal={arXiv preprint arXiv:2601.20802},
  year={2026}
}

@article{lu2025onpolicydistillation,
  author = {Kevin Lu and Thinking Machines Lab},
  title = {On-Policy Distillation},
  journal = {Thinking Machines Lab: Connectionism},
  year = {2025},
  note = {https://thinkingmachines.ai/blog/on-policy-distillation},
  doi = {10.64434/tml.20251026},
}

@article{damani2025beyond,
  title={Beyond binary rewards: Training lms to reason about their uncertainty},
  author={Damani, Mehul and Puri, Isha and Slocum, Stewart and Shenfeld, Idan and Choshen, Leshem and Kim, Yoon and Andreas, Jacob},
  journal={arXiv preprint arXiv:2507.16806},
  year={2025}
}

@article{leng2024taming,
  title={Taming Overconfidence in {LLM}s: Reward Calibration in {RLHF}},
  author={Leng, Jixuan and Huang, Chengsong and Zhu, Banghua and Huang, Jiaxin},
  journal={arXiv preprint arXiv:2410.09724},
  year={2024}
}

@inproceedings{guo2017calibration,
  title={On calibration of modern neural networks},
  author={Guo, Chuan and Pleiss, Geoff and Sun, Yu and Weinberger, Kilian Q},
  booktitle={International conference on machine learning},
  pages={1321--1330},
  year={2017},
  organization={PMLR}
}

@article{feng2024sciknoweval,
  title={Sciknoweval: Evaluating multi-level scientific knowledge of large language models},
  author={Feng, Kehua and Shen, Xinyi and Wang, Weijie and Zhuang, Xiang and Tang, Yuqi and Zhang, Qiang and Ding, Keyan},
  journal={arXiv preprint arXiv:2406.09098},
  year={2024}
}

@article{tang2023toolalpaca,
  title={Toolalpaca: Generalized tool learning for language models with 3000 simulated cases},
  author={Tang, Qiaoyu and Deng, Ziliang and Lin, Hongyu and Han, Xianpei and Liang, Qiao and Cao, Boxi and Sun, Le},
  journal={arXiv preprint arXiv:2306.05301},
  year={2023}
}

@article{olmo2025olmo,
  title={Olmo 3},
  author={Olmo, Team and Ettinger, Allyson and Bertsch, Amanda and Kuehl, Bailey and Graham, David and Heineman, David and Groeneveld, Dirk and Brahman, Faeze and Timbers, Finbarr and Ivison, Hamish and others},
  journal={arXiv preprint arXiv:2512.13961},
  year={2025}
}

@article{shao2024deepseekmath,
  title={Deepseekmath: Pushing the limits of mathematical reasoning in open language models},
  author={Shao, Zhihong and Wang, Peiyi and Zhu, Qihao and Xu, Runxin and Song, Junxiao and Bi, Xiao and Zhang, Haowei and Zhang, Mingchuan and Li, YK and Wu, Yang and others},
  journal={arXiv preprint arXiv:2402.03300},
  year={2024}
}

@article{kalai2025language,
  title={Why language models hallucinate},
  author={Kalai, Adam Tauman and Nachum, Ofir and Vempala, Santosh S and Zhang, Edwin},
  journal={arXiv preprint arXiv:2509.04664},
  year={2025}
}

@inproceedings{manakul2023selfcheckgpt,
  title={Selfcheckgpt: Zero-resource black-box hallucination detection for generative large language models},
  author={Manakul, Potsawee and Liusie, Adian and Gales, Mark},
  booktitle={Proceedings of the 2023 conference on empirical methods in natural language processing},
  pages={9004--9017},
  year={2023}
}

@inproceedings{zhang2023sac3,
  title={SAC3: reliable hallucination detection in black-box language models via semantic-aware cross-check consistency},
  author={Zhang, Jiaxin and Li, Zhuohang and Das, Kamalika and Malin, Bradley and Kumar, Sricharan},
  booktitle={Findings of the Association for Computational Linguistics: EMNLP 2023},
  pages={15445--15458},
  year={2023}
}

@article{zhao2026self,
  title={Self-Distilled Reasoner: On-Policy Self-Distillation for Large Language Models},
  author={Zhao, Siyan and Xie, Zhihui and Liu, Mengchen and Huang, Jing and Pang, Guan and Chen, Feiyu and Grover, Aditya},
  journal={arXiv preprint arXiv:2601.18734},
  year={2026}
}

@inproceedings{zhang2026passive,
  title={From Passive Metric to Active Signal: The Evolving Role of Uncertainty Quantification in Large Language Models},
  author={Zhang, Jiaxin and Cui, Wendi and Li, Zhuohang and Huang, Lifu and Malin, Bradley A and Xiong, Caiming and Wu, Chien-Sheng},
  booktitle={Findings of the Association for Computational Linguistics: ACL 2026},
  pages={41525--41544},
  year={2026}
}

@article{yang2026calibration,
  title={On calibration of large language models: From response to capability},
  author={Yang, Sin-Han and Wu, Cheng-Kuang and Lin, Chieh-Yen and Chen, Yun-Nung and Lee, Hung-yi and Sun, Shao-Hua},
  journal={arXiv preprint arXiv:2602.13540},
  year={2026}
}

@article{schulman2017proximal,
  title={Proximal policy optimization algorithms},
  author={Schulman, John and Wolski, Filip and Dhariwal, Prafulla and Radford, Alec and Klimov, Oleg},
  journal={arXiv preprint arXiv:1707.06347},
  year={2017}
}

@article{cui2026rethinking,
  title={Rethinking Agentic Reinforcement Learning In Large Language Models},
  author={Cui, Fangming and Zhu, Ruixiao and Fang, Cheng and Li, Sunan and Li, Jiahong},
  journal={arXiv preprint arXiv:2604.27859},
  year={2026}
}

@article{ding2026hdpo,
  title={HDPO: Hybrid Distillation Policy Optimization via Privileged Self-Distillation},
  author={Ding, Ken},
  journal={arXiv preprint arXiv:2603.23871},
  year={2026}
}

@article{wang2026tcod,
  title={TCOD: Exploring Temporal Curriculum in On-Policy Distillation for Multi-turn Autonomous Agents},
  author={Wang, Jiaqi and Zhang, Wenhao and Shi, Weijie and Li, Yaliang and Cheng, James},
  journal={arXiv preprint arXiv:2604.24005},
  year={2026}
}

@article{song2026survey,
  title={A Survey of On-Policy Distillation for Large Language Models},
  author={Song, Mingyang and Zheng, Mao},
  journal={arXiv preprint arXiv:2604.00626},
  year={2026}
}

@article{dong2025agentic,
  title={Agentic reinforced policy optimization},
  author={Dong, Guanting and Mao, Hangyu and Ma, Kai and Bao, Licheng and Chen, Yifei and Wang, Zhongyuan and Chen, Zhongxia and Du, Jiazhen and Wang, Huiyang and Zhang, Fuzheng and others},
  journal={arXiv preprint arXiv:2507.19849},
  year={2025}
}

@article{zhang2025landscape,
  title={The landscape of agentic reinforcement learning for llms: A survey},
  author={Zhang, Guibin and Geng, Hejia and Yu, Xiaohang and Yin, Zhenfei and Zhang, Zaibin and Tan, Zelin and Zhou, Heng and Li, Zhongzhi and Xue, Xiangyuan and Li, Yijiang and others},
  journal={arXiv preprint arXiv:2509.02547},
  year={2025}
}

@article{oh2026uncertainty,
  title={Uncertainty Quantification in LLM Agents: Foundations, Emerging Challenges, and Opportunities},
  author={Oh, Changdae and Park, Seongheon and Kim, To Eun and Li, Jiatong and Li, Wendi and Yeh, Samuel and Du, Xuefeng and Hassani, Hamed and Bogdan, Paul and Song, Dawn and others},
  journal={arXiv preprint arXiv:2602.05073},
  year={2026}
}

@article{zhang2026agentic,
  title={Agentic Uncertainty Quantification},
  author={Zhang, Jiaxin and Choubey, Prafulla Kumar and Huang, Kung-Hsiang and Xiong, Caiming and Wu, Chien-Sheng},
  journal={arXiv preprint arXiv:2601.15703},
  year={2026}
}

@article{zhang2026acc,
  title={Agentic confidence calibration},
  author={Zhang, Jiaxin and Xiong, Caiming and Wu, Chien-Sheng},
  journal={arXiv preprint arXiv:2601.15778},
  year={2026}
}

@article{zhang2026illusion,
  title={The Illusion of Certainty: Decoupling Capability and Calibration in On-Policy Distillation},
  author={Zhang, Jiaxin and Peng, Xiangyu and Chen, Qinglin and Ye, Qinyuan and Xiong, Caiming and Wu, Chien-Sheng},
  journal={arXiv preprint arXiv:2604.16830},
  year={2026}
}

@article{vapnik2009new,
  title={A new learning paradigm: Learning using privileged information},
  author={Vapnik, Vladimir and Vashist, Akshay},
  journal={Neural networks},
  volume={22},
  number={5-6},
  pages={544--557},
  year={2009},
  publisher={Elsevier}
}

@article{lopez2015unifying,
  title={Unifying distillation and privileged information},
  author={Lopez-Paz, David and Bottou, L{\'e}on and Sch{\"o}lkopf, Bernhard and Vapnik, Vladimir},
  journal={arXiv preprint arXiv:1511.03643},
  year={2015}
}

@article{agarwal2025gptoss,
  title={gpt-oss-120b \& gpt-oss-20b model card},
  author={Agarwal, Sandhini and Ahmad, Lama and Ai, Jason and Altman, Sam and Applebaum, Andy and Arbus, Edwin and Arora, Rahul K and Bai, Yu and Baker, Bowen and Bao, Haiming and others},
  journal={arXiv preprint arXiv:2508.10925},
  year={2025}
}

%%%%%%%%%%%%%%%%%%%%%%%%%%%%%%%%%%%%%%%%%%%%%%%%%%%%%%%%%%%%

% ==========================================
% 生成全局目录 (Table of Contents)
% ==========================================
\clearpage
\begin{center}
    \Large \textbf{Contents}
\end{center}
% 如果你希望目录里的链接是黑色的而不是蓝色的，可以取消下面这行的注释
% \hypersetup{linkcolor=black} 
\tableofcontents
\clearpage
% ==========================================

\appendix

% ---------------------------------------------------------
% \clearpage
% \appendix
% \section{Appendix}
% You may include other additional sections here.

%==============================================================================
\section{Extended Related Work and Discussion}
\label{app:related_extended}
%==============================================================================

\paragraph{Agent RL: from outcome rewards to process supervision.}
\label{app:related_agent_rl}
Reinforcement learning for language-model agents has progressed from
sparse outcome-only rewards (e.g., GRPO~\cite{shao2024deepseekmath}
applied at the trajectory level) to dense, step-level supervision
delivered by process reward models and structured frameworks.
OpenClaw-RL~\cite{wang2026openclaw} provides a unified asynchronous
training stack that combines binary RL, on-policy distillation, and their
composition in one rollout--score--update loop, together with a per-turn
PRM that grades each turn from the next environmental state. Buening et
al.~\cite{buening2026aligning} push this further by showing that raw
multi-turn user interactions are already sufficient supervision for
self-distillation, eliminating the need for a separately trained reward
model. Across these systems, the per-turn signal varies considerably in
fidelity and modality: rule-based binary verifiers for verifiable
single-turn tasks, learned scalar PRMs for soft multi-turn signals, and
hindsight-teacher token distributions for OPD \cite{lu2025onpolicydistillation,zhao2026self}. What does \emph{not} vary
is the structural property that motivates this paper: every learning
signal is computed from quantities observable only after the action is
committed, and the agent's prospective belief
at action time is invisible to the gradient.
Prospective Hindsight is orthogonal to the choice of base method, leaves
the post-action signal untouched, and adds a single per-rollout
reweighting on top.

\paragraph{Hindsight distillation and the OPD lineage.}
\label{app:related_opd}
On-policy distillation~\cite{agarwal2024policy} introduced the idea of
distilling a teacher's improved actions into the student under on-policy
rollouts; subsequent work has extended this in several directions. Ye et
al.~\cite{ye2026policy} formalize on-policy context distillation, in
which the teacher conditions on a privileged context $c_t$ that the
student does not see at deployment;
Hübotter et al.~\cite{hubotter2026reinforcement} and Shenfeld et
al.~\cite{shenfeld2026self} develop self-distillation variants in which
the teacher and student share parameters and the privileged context is
constructed on-the-fly from verified-correct rollouts. We position PH as
a complementary mechanism in this lineage. Standard OPD \cite{lu2025onpolicydistillation} uses privileged
information to construct \emph{better actions} for the teacher, and
the resulting per-token signal is the dominant training signal. PH instead
uses the gap between privileged and non-privileged \emph{evaluations} to
identify which rollouts deserve more weight, and absorbs OPD's per-token
signal into a single per-rollout multiplier.
The two mechanisms can be applied jointly without interaction: in our
multi-turn experiments PH is plugged on top of GRPO, OPD, and the
OpenClaw-RL Combined recipe with no method-specific tuning.

\paragraph{Calibration and self-awareness in LLMs.}
\label{app:related_calibration}
Modern confidence calibration in LLMs draws on several lines of work. At
the diagnostic level, Guo et al.~\cite{guo2017calibration} pioneered the
study of miscalibrated softmax outputs in deep networks, and Kadavath et
al.~\cite{kadavath2022language-auq} systematically measured whether
language models ``know what they know'' across knowledge tasks. Methods
to elicit calibrated confidence include verbalized confidence
prompts~\cite{tian2023just-auq,lin2022teaching-acc} and sampling-based
estimators of semantic uncertainty~\cite{kuhnsemantic-auq}; Geng et
al.~\cite{geng2024survey-acc} provide a survey. At the agentic level,
Kirchhof et al.~\cite{kirchhofposition-auq} argue for trajectory-level
calibration diagnostics, and a series of recent
papers~\cite{duan2025uprop-auq,zhao2025uncertainty-auq,liu2024uncertainty-acc, zhang2026passive}
quantify uncertainty propagation through multi-step trajectories;
Han et al.~\cite{han2024towards-auq} use uncertainty as an inference-time
control signal. Kalai et al.~\cite{kalai2025language} take a different
angle and argue that hallucination itself is incentivized by training
procedures that reward confident guesses over calibrated abstention,
linking training-procedure design to overconfidence at the model-output
level. This literature largely treats calibration as a diagnostic
property to measure or to control at inference, rather than as a training
signal to optimize. PH inverts that relationship: it converts the binary
self-evaluation mismatch into a per-rollout training-time weight, and the
binary Brier score of the agent's self-evaluator coincides with the
miscalibration rate that the PH residual
lower-bounds. Calibration emerges as a
side effect of the same gradient steps that drive task success, rather
than as a post-hoc check or a separate auxiliary loss.

\paragraph{Calibration-augmented RL and the capability--calibration trade-off.}
\label{app:related_overconfidence}
A line of recent work specifically targets the overconfidence introduced
by hindsight distillation. CaOPD~\cite{zhang2026illusion} attributes this
overconfidence to the information asymmetry between the privileged-context
teacher and the student and proposes \emph{explicit confidence-token
replacement} in the teacher distribution; Leng et al.~\cite{leng2024taming}
and Damani et al.~\cite{damani2025beyond} study calibration-augmented RL
recipes that add explicit confidence supervision to the loss. Across these
approaches, calibration improvements typically come at the cost of a
modest capability regression, and the calibration objective is a separate
term added to the base RL loss. PH is complementary along two dimensions.
First, the base method is left untouched: PH does not rewrite confidence
targets, does not add a separate calibration loss, and does not require a
calibration-specific reward model. Second, the per-rollout reweighting
connects calibration and capability through the surprise residual
decomposition: a gradient step that reduces the
base loss on miscalibrated rollouts and a gradient step that reduces the
miscalibration rate are both valid descent directions on. Empirically, this removes the
capability--calibration trade-off seen in the cited baselines: PH
preserves capability while substantially cutting OFR on single-turn and lifts both the trajectory-averaged and
best-checkpoint scores across all three multi-turn base methods.

\paragraph{Privileged information and prediction-error signals.}
\label{app:related_lupi}
Our framing of the prediction--reality gap connects directly to Vapnik's
Learning Under Privileged Information
framework~\cite{vapnik2009new,pechyony2010theory}. In LUPI, training
examples are accompanied by privileged information that is not available
at deployment, and the goal is a predictor that uses standard inputs
alone. Lopez-Paz et al.~\cite{lopez2015unifying} unify LUPI with
knowledge distillation, showing that distillation can be viewed as a
transfer of privileged information from teacher to student. Our agent-RL
setting instantiates LUPI with $(s_t, a_t)$ as the standard input, $y_t$
as the target, and $c_t$ as the privileged signal. Standard OPD-style
methods use $c_t$ to construct improved actions for the teacher; PH
instead uses the \emph{evaluation} gap that $c_t$ creates, formalized as
the non-negative privileged-information gap $I^{\star}(s_t,a_t)$ in
Proposition~\ref{prop:gap}. To our knowledge this is the first explicit
use of $I^{\star}$ as a sample-level training signal in agent RL: the
surprise indicator $\xi_t$ is the per-rollout realization of $I^{\star}$
from inside the agent's own model, and upweighting surprising rollouts
concentrates gradient budget exactly where action-time information is
insufficient.

A related but distinct line is intrinsic-motivation RL, in which
prediction errors between two predictors of the next state drive
\emph{exploration} of novel states. PH differs in three ways. First, the
prediction error is computed in the \emph{outcome / evaluation} space,
not the state space, and is used for training-signal prioritization
rather than for action selection. Second, the predictor is the same model
as the policy, parameterized by the same $\theta$ and queried via a
different prompt, rather than a separate frozen network. Third, the
surprise scalar is stop-gradient: it modulates the magnitude of the base
gradient on the existing on-policy rollouts but does not introduce a new
gradient component of its own. The two perspectives are complementary
--- one amplifies novelty in state space, the other amplifies prediction
errors in evaluation space --- and we view their combination as a natural
avenue for future work.

\paragraph{Self-correction, self-refinement, and reasoning chains.}
\label{app:related_selfcorrect}
A parallel line of work uses model self-evaluation to drive
\emph{inference-time} behavior: Reflexion~\cite{shinn2023reflexion-auq}
and Self-Refine~\cite{madaan2023self-auq} iterate on a draft response
using a self-critic. PH operates at training time rather than inference
time; the prospective prediction is consumed once, used to compute
$\xi_t$, and discarded after the gradient step. This is a deliberate
design choice: PH does not change deployment behavior, does not lengthen
the agent's effective context, and does not require the agent to reason
about its own confidence at inference. Inference-time self-correction is
therefore composable with PH (it operates on the policy distribution
after PH has shaped it), but the two mechanisms target different stages
of the pipeline. A second related line is hallucination characterization
(e.g.,~\cite{manakul2023selfcheckgpt,zhang2023sac3,zhao2026self}), which
studies the conditions under which LLMs produce unsupported content; we
view PH's reduction of overconfident-failure rate as operating in the
same regime, but at the agent-RL training-loop level rather than at the
per-output classification level.

\section{Theoretical Details and Proofs}
\label{app:theo}

\subsection{Proof of Proposition~\ref{prop:gap}}
\label{app:gap_proof}

Fix $(s_t, a_t)$ and apply the law of total variance to $y_t$ with conditioning random variable $c_t$:
\begin{align}
    \mathrm{Var}(y_t \mid s_t, a_t)
    &= \mathbb{E}\!\left[\mathrm{Var}(y_t \mid s_t, a_t, c_t) \,\big|\, s_t, a_t\right] \nonumber \\
    &\quad + \mathrm{Var}\!\left(\mathbb{E}[y_t \mid s_t, a_t, c_t] \,\big|\, s_t, a_t\right) \nonumber \\
    &= \mathbb{E}\!\left[\mathrm{Var}(y_t \mid s_t, a_t, c_t) \,\big|\, s_t, a_t\right] + I^{\star}(s_t, a_t),
\end{align}
where the second equality uses the definition of $I^{\star}$. Both terms are non-negative, and the second vanishes if and only if $\mathbb{E}[y_t \mid s_t, a_t, c_t]$ is almost surely a function of $(s_t, a_t)$, that is, $c_t$ provides no additional information about $y_t$ beyond $(s_t, a_t)$. \hfill$\square$

\paragraph{Connection to Learning Under Privileged Information}
\label{app:lupi}

The decomposition in Proposition~\ref{prop:gap} places the agent RL setting of \S\ref{sec:agent_rl} inside the Learning Under Privileged Information (LUPI) paradigm introduced by \cite{vapnik2009new}. In LUPI, training examples are accompanied by privileged information available only at training time, and the goal is a predictor that uses standard inputs alone. \cite{pechyony2010theory} formalize the statistical-learning theory of this setting. \cite{lopez2015unifying} unify LUPI with knowledge distillation, showing that distillation can be viewed as a transfer of privileged information from teacher to student.

Our setting instantiates LUPI with $(s_t, a_t)$ as the standard input, $y_t$ as the target, and $c_t$ as the privileged signal. The base methods of \S\ref{sec:retro_methods} use $c_t$ in different ways: GRPO uses $c_t$ only through the scalar reward $r_t = R(s_t, a_t, c_t)$; OPD-style methods condition the hindsight teacher distribution on $c_t$ and distill to a student that drops $c_t$ at inference. Both retain the blind spot of \S\ref{sec:blind_spot}: training never asks the policy to predict $y_t$ from $(s_t, a_t)$ before $c_t$ arrives, so the privileged-information gap $I^{\star}(s_t, a_t)$ is never measured per step. Our method (\S\ref{sec:method}) makes this gap an explicit training signal.

\paragraph{Equivalent information-theoretic statement.}
Equation~\eqref{eq:gap_decomp} can equivalently be expressed in entropy form. For binary $y_t$, let $H(y \mid \cdot) := -\mathbb{E}[\log p(y \mid \cdot)]$. Then
\begin{equation}
    H\!\left(y_t \mid s_t, a_t\right) - H\!\left(y_t \mid s_t, a_t, c_t\right) \;=\; I\!\left(y_t; c_t \,\big|\, s_t, a_t\right) \;\geq\; 0,
    \label{eq:mi_form}
\end{equation}
that is, the conditional mutual information between $y_t$ and $c_t$ given $(s_t, a_t)$ measures the same privileged-information increment as $I^{\star}$. The equality $I(y_t; c_t \mid s_t, a_t) = 0$ characterizes the degenerate case in which $I^{\star}(s_t, a_t) = 0$ for almost every $(s_t, a_t)$: no prospective predictor can be improved by access to $c_t$ at training time.

\subsection{PH as Calibration through Reweighting: Detailed Analysis}
\label{app:calibration}

This appendix provides the formal setup, the exact residual identity
and its consequences, a gradient-flow analysis of the PH update
direction, and a discussion of the assumptions' empirical validity.
We frame the results below as \emph{motivation} for the calibration
interpretation of PH outlined in \S\ref{sec:calibration_view}, not as
formal guarantees of monotone calibration descent. 

\subsubsection*{Setup and notation}
\label{app:cal:setup}

We work in a single-step task to keep notation light; the multi-turn
extension is recorded in \S\ref{app:cal:discussion}. Let
$\theta \in \Theta$ parameterize a shared autoregressive model that
defines:
\begin{itemize}[leftmargin=*]
    \item the policy $\pi_\theta(a \mid s)$, a conditional distribution
    over actions $a \in \mathcal{A}$ given state $s \in \mathcal{S}$;
    \item the prospective self-evaluator
    $p^{\mathrm{eval}}_\theta(z \mid s, a)$, a conditional distribution
    over evaluator outputs $z \in \{-1, 0, +1\}$.
\end{itemize}
The verifier $y$ is sampled from a parameter-independent kernel
$p_{\mathrm{env}}(y \mid s, a)$ with $y \in \{0, 1\}$; the state
distribution is $\rho_0(s)$. The joint rollout distribution induced by
$\theta$ is
\begin{equation}
    \rho_\theta(s, a, z, y) \;=\; \rho_0(s)\, \pi_\theta(a \mid s)\, p^{\mathrm{eval}}_\theta(z \mid s, a)\, p_{\mathrm{env}}(y \mid s, a).
    \label{eq:joint_rho_app}
\end{equation}
We restrict attention to $z \in \{-1, +1\}$; the residual $z = 0$ class
is treated separately in App.~\ref{app:hu} and contributes neutrally to
all expressions below since $\xi_t = 0$ there. The surprise indicator
(Eq.~\ref{eq:surprise} of the main text) is
\begin{equation}
    \xi(z, y) \;=\; \mathbb{1}\!\bigl[(z, y) \in \{\mathrm{OF}, \mathrm{US}\}\bigr],
\end{equation}
and we write $\xi_t$ when the rollout is implicit. We assume the base
loss $\ell^{\mathrm{base}}_t(\theta)$ is non-negative; this is satisfied
by all retrospective methods of \S\ref{sec:retro_methods} after additive
constant shifts that do not affect the gradient.

\subsubsection*{Surprise-weighted decomposition}
\label{app:cal:decomp}

\begin{proposition}[Surprise-weighted decomposition]
\label{prop:ph_decomp}
For any $\alpha \geq 0$,
\begin{equation*}
    \mathbb{E}_{\rho_\theta}\!\bigl[\ell^{\mathrm{PH}}_t\bigr] \;=\; \mathbb{E}_{\rho_\theta}\!\bigl[\ell^{\mathrm{base}}_t\bigr] + \alpha \cdot \mathcal{R}(\theta), \quad \mathcal{R}(\theta) := \mathbb{E}_{\rho_\theta}\!\bigl[\xi_t \cdot \ell^{\mathrm{base}}_t\bigr].
\end{equation*}
The surprise residual $\mathcal{R}(\theta)$ is non-negative and vanishes
if and only if $\xi_t \cdot \ell^{\mathrm{base}}_t = 0$ holds
$\rho_\theta$-almost surely.
\end{proposition}

\begin{proof}
By Eq.~\ref{eq:ph_loss} of the main text,
$\ell^{\mathrm{PH}}_t = (1 + \alpha \xi_t)\, \ell^{\mathrm{base}}_t = \ell^{\mathrm{base}}_t + \alpha \xi_t \ell^{\mathrm{base}}_t$.
Linearity of expectation gives the decomposition. Non-negativity of
$\mathcal{R}$ follows from $\xi_t \in \{0, 1\}$ and
$\ell^{\mathrm{base}}_t \geq 0$. The vanishing condition is the standard
a.s.\ zero condition for non-negative integrands.
\end{proof}

The two pathways for residual minimization---reducing $\Pr[\xi = 1]$
(calibration) or reducing $\ell^{\mathrm{base}}$ on miscalibrated
rollouts (base-objective improvement on those rollouts)---are
formalized exactly in the next subsection.

\subsubsection*{Exact residual identity}
\label{app:cal:bound}

\begin{proposition}[Exact residual identity]
\label{prop:residual_miscal}
Define the agent's miscalibration rate
\begin{equation}
    \mathcal{M}(\theta) \;:=\; \mathbb{E}_{\rho_\theta}\!\bigl[\xi_t\bigr] \;=\; \Pr\nolimits_{\rho_\theta}\!\bigl[(z_t, y_t) \in \{\mathrm{OF}, \mathrm{US}\}\bigr],
    \label{eq:miscal_app}
\end{equation}
and the conditional mean base loss on surprising rollouts
\begin{equation}
    c(\theta) \;:=\; \mathbb{E}_{\rho_\theta}\!\bigl[\ell^{\mathrm{base}}_t \,\big|\, \xi_t = 1\bigr].
\end{equation}
Then
\begin{equation}
    \mathcal{R}(\theta) \;=\; c(\theta) \cdot \mathcal{M}(\theta).
    \label{eq:exact_residual_app}
\end{equation}
\end{proposition}

\begin{proof}
By the tower rule,
\begin{align*}
    \mathcal{R}(\theta) &\;=\; \mathbb{E}\!\bigl[\xi_t \ell^{\mathrm{base}}_t\bigr]
    \;=\; \mathbb{E}\!\bigl[\ell^{\mathrm{base}}_t \,\big|\, \xi_t = 1\bigr] \cdot \Pr\bigl[\xi_t = 1\bigr]
    \;=\; c(\theta) \cdot \mathcal{M}(\theta). \qedhere
\end{align*}
\end{proof}

Eq.~\ref{eq:exact_residual_app} is an \emph{exact} identity, not an
inequality: the residual factors cleanly into a base-loss-on-miscalibrated
component $c(\theta)$ and a miscalibration-rate component
$\mathcal{M}(\theta)$. As a corollary, whenever $c(\theta) \geq c_0 > 0$
on a contiguous training window, $\mathcal{R}(\theta) \geq c_0 \cdot \mathcal{M}(\theta)$
on that window, and minimizing $\mathcal{R}(\theta)$ via PH must reduce
$\mathcal{M}(\theta)$ at rate at least $c_0$ relative to the rate of
residual descent.

\paragraph{When does Eq.~\ref{eq:exact_residual_app} drive calibration?}
The identity is most informative as a calibration lever when both
factors are non-negligible. Three regimes arise as training progresses:
\begin{itemize}[leftmargin=1.6em]
    \item \textbf{Early training.} $c(\theta)$ is large (the base
    method has not yet learned to solve surprising rollouts) and
    $\mathcal{M}(\theta)$ is large (many rollouts are miscalibrated).
    Both pathways for residual minimization are active, and PH's
    surprise reweighting allocates additional gradient budget to
    cells where the prediction--reality gap is widest.
    \item \textbf{Mid training.} $c(\theta)$ and $\mathcal{M}(\theta)$
    decrease together as the agent's self-model and the base policy
    co-evolve. The dual-pathway decomposition continues to hold, but
    the per-step weight of each pathway shifts.
    \item \textbf{Convergence.} $c(\theta) \to 0$ and
    $\mathcal{M}(\theta) \to 0$ jointly, $\mathcal{R}(\theta) \to 0$,
    and PH and the base method share the same fixed points
    (App.~\ref{app:cal:discussion}). The identity becomes
    uninformative because both sides are zero, but PH itself is no
    longer doing work in this regime.
\end{itemize}
The bound $\mathcal{R}(\theta) \geq c_0 \cdot \mathcal{M}(\theta)$ is
therefore not a uniform property of the optimization landscape: it
characterizes the operative training window during which PH's
calibration gains accrue. 

\paragraph{Equivalence to the binary Brier score.}
For binary $z'_t = (z_t + 1)/2 \in \{0, 1\}$ and $y_t \in \{0, 1\}$,
$(z'_t - y_t)^2 = \mathbb{1}(z'_t \neq y_t) = \xi_t$, so
$\mathcal{M}(\theta)$ coincides with the standard binary Brier score
$\mathbb{E}[(z'_t - y_t)^2]$ of the deterministic predictor $z'_t$
against the verifier $y_t$. We use the term \emph{miscalibration rate}
throughout because it more directly conveys the operational
meaning---the fraction of rollouts on which the agent's prospective
belief disagrees with the verifier outcome---and because the
experimental section reports a one-sided projection of $\mathcal{M}$
rather than the full Brier score (\S\ref{app:cal:ofr}).

\subsubsection*{Connection to the overconfident-failure rate (OFR)}
\label{app:cal:ofr}

The miscalibration rate $\mathcal{M}(\theta)$ in Eq.~\ref{eq:miscal_app}
aggregates both miscalibrated cells, OF and US, symmetrically. The
experimental section reports the OF-conditional projection
\begin{equation}
    \mathrm{OFR}(\theta) \;:=\; \Pr\nolimits_{\rho_\theta}\!\bigl[y_t = 0 \,\big|\, z_t = +1\bigr] \;=\; \frac{\Pr[\mathrm{OF}]}{\Pr[\mathrm{CS}] + \Pr[\mathrm{OF}]} \;=\; \frac{\Pr[\mathrm{OF}]}{\Pr[z_t = +1]}.
    \label{eq:ofr_app}
\end{equation}
This is a one-sided, conditional version of $\mathcal{M}$ focused on
the OF cell, which is the dominant miscalibration mode singled out in
\S\ref{sec:blind_spot}. The two metrics are related by
\begin{equation}
    \Pr[\mathrm{OF}] \;\leq\; \mathcal{M}(\theta), \qquad \mathrm{OFR}(\theta) \;=\; \Pr[\mathrm{OF}] \,/\, \Pr[z_t = +1],
\end{equation}
so reductions in $\mathcal{M}(\theta)$ imply reductions in
$\Pr[\mathrm{OF}]$ whenever $\Pr[\mathrm{US}]$ does not increase, and
reductions in $\Pr[\mathrm{OF}]$ imply reductions in $\mathrm{OFR}(\theta)$
whenever $\Pr[z_t = +1]$ does not collapse. Both monotonicity assumptions
hold in our experiments: the prospective predictor's confidence rate
$\Pr[z_t = +1]$ stays in a stable range across training, and the
underconfident-success rate $\Pr[\mathrm{US}]$ remains a small fraction
of $\mathcal{M}(\theta)$ at all checkpoints. We therefore report
$\mathrm{OFR}(\theta)$ as a faithful empirical proxy for $\mathcal{M}(\theta)$
in \S\ref{sec:experiments}.

\subsubsection*{Gradient-flow analysis: when is the PH update calibration-aligned?}
\label{app:cal:gradient}
\label{app:variance}

We now characterize the direction in parameter space along which PH
gradient flow moves, relative to the negative gradient of the
miscalibration rate. The result below should be read as a
\emph{motivation} for why one would expect PH to reduce
$\mathcal{M}(\theta)$, not as a guarantee that it does so monotonically
in every step.

Let $u^{\mathrm{PH}} := -\nabla_\theta \mathbb{E}_{\rho_\theta}[\ell^{\mathrm{PH}}]$
and $u^{\mathrm{base}} := -\nabla_\theta \mathbb{E}_{\rho_\theta}[\ell^{\mathrm{base}}]$
denote the PH and base update directions. By
Prop.~\ref{prop:ph_decomp},
\begin{equation}
    u^{\mathrm{PH}} \;=\; u^{\mathrm{base}} - \alpha \nabla_\theta \mathcal{R}(\theta).
    \label{eq:u_decomp}
\end{equation}

\paragraph{Score-function decomposition of $\nabla_\theta \mathcal{R}$.}
Applying the score-function trick to
$\mathcal{R}(\theta) = \mathbb{E}_{\rho_\theta}[\xi_t \ell^{\mathrm{base}}_t]$,
\begin{equation}
    \nabla_\theta \mathcal{R}(\theta) \;=\; \underbrace{\mathbb{E}_{\rho_\theta}[\xi_t \nabla_\theta \ell^{\mathrm{base}}_t]}_{\text{base-loss descent on }\{\xi=1\}} \;+\; \underbrace{\mathbb{E}_{\rho_\theta}[\xi_t \ell^{\mathrm{base}}_t \nabla_\theta \log \rho_\theta]}_{\text{distributional shift in }\xi},
    \label{eq:R_grad_decomp}
\end{equation}
where $\nabla_\theta \log \rho_\theta = \nabla_\theta \log \pi_\theta(a \mid s) + \nabla_\theta \log p^{\mathrm{eval}}_\theta(z \mid s, a)$,
since $\rho_0$ and $p_{\mathrm{env}}$ are parameter-free. Similarly,
\begin{equation}
    \nabla_\theta \mathcal{M}(\theta) \;=\; \mathbb{E}_{\rho_\theta}[\xi_t \nabla_\theta \log \rho_\theta].
    \label{eq:miscal_grad}
\end{equation}

\begin{theorem}[PH gradient under uniform-loss assumptions]
\label{thm:ph_miscal}
Let $\bar{\ell}(\theta) := \mathbb{E}_{\rho_\theta}[\ell^{\mathrm{base}}_t \mid \xi_t = 1] = c(\theta)$
and $\delta(\theta) := \mathrm{Var}_{\rho_\theta,\, \xi_t = 1}[\ell^{\mathrm{base}}_t]$.
Suppose $\mathbb{E}[\xi_t \|\nabla_\theta \log \rho_\theta\|^2] < \infty$.
Then
\begin{equation}
    u^{\mathrm{PH}} \;=\; u^{\mathrm{base}} \;-\; \alpha \cdot \mathbb{E}_{\rho_\theta}[\xi_t \nabla_\theta \ell^{\mathrm{base}}_t] \;-\; \alpha \, \bar{\ell}(\theta) \cdot \nabla_\theta \mathcal{M}(\theta) \;+\; r(\theta),
    \label{eq:thm_main}
\end{equation}
where the remainder $r(\theta)$ satisfies
\begin{equation}
    \|r(\theta)\| \;\leq\; \alpha \sqrt{\delta(\theta)} \cdot \sqrt{\mathbb{E}[\xi_t \|\nabla_\theta \log \rho_\theta\|^2]}.
    \label{eq:thm_remainder}
\end{equation}
In the limit $\delta(\theta) \to 0$ (the base loss is constant on
surprising rollouts), the PH update direction differs from the base
direction by exactly a $-\alpha \bar{\ell} \cdot \nabla_\theta \mathcal{M}(\theta)$
component plus a base-loss descent component on the miscalibrated set.
\end{theorem}

\begin{proof}
Decompose the second term of Eq.~\ref{eq:R_grad_decomp} via the identity
\begin{equation*}
    \mathbb{E}[\xi_t \ell^{\mathrm{base}}_t \nabla_\theta \log \rho_\theta] \;=\; \bar{\ell} \cdot \mathbb{E}[\xi_t \nabla_\theta \log \rho_\theta] \;+\; \mathbb{E}\bigl[\xi_t (\ell^{\mathrm{base}}_t - \bar{\ell}) \nabla_\theta \log \rho_\theta\bigr].
\end{equation*}
The first term equals $\bar{\ell} \cdot \nabla_\theta \mathcal{M}(\theta)$
by Eq.~\ref{eq:miscal_grad}. The second term is bounded in norm by
Cauchy--Schwarz:
\begin{equation*}
    \bigl\| \mathbb{E}[\xi_t (\ell^{\mathrm{base}}_t - \bar{\ell}) \nabla_\theta \log \rho_\theta] \bigr\| \;\leq\; \sqrt{\mathrm{Var}_{\xi_t = 1}[\ell^{\mathrm{base}}_t]} \cdot \sqrt{\mathbb{E}[\xi_t \|\nabla_\theta \log \rho_\theta\|^2]} \;=\; \sqrt{\delta(\theta)} \cdot \sqrt{\mathbb{E}[\xi_t \|\nabla_\theta \log \rho_\theta\|^2]},
\end{equation*}
which gives Eq.~\ref{eq:thm_remainder} after multiplication by $\alpha$.
Substituting back into Eqs.~\ref{eq:R_grad_decomp} and \ref{eq:u_decomp}
yields Eq.~\ref{eq:thm_main}.
\end{proof}

\begin{remark}[Empirical magnitude of the remainder]
\label{rem:delta_empirical}
The remainder bound in Eq.~\ref{eq:thm_remainder} is not an assumption
but a quantity measurable from training logs. We define the relative
remainder
\begin{equation}
    \rho(\theta) \;:=\; \frac{\sqrt{\delta(\theta)}}{c(\theta)}
    \;=\; \frac{\sqrt{\mathrm{Var}_{\xi_t=1}[\ell^{\mathrm{base}}_t]}}{\mathbb{E}_{\xi_t=1}[\ell^{\mathrm{base}}_t]},
\end{equation}
the coefficient of variation of the base loss conditional on
$\xi_t = 1$. Theorem~\ref{thm:ph_miscal} is most informative when
$\rho(\theta)$ is small; the limit $\rho(\theta) \to 0$ recovers the
uniform-loss case in which the leading term
$-\alpha \bar{\ell} \cdot \nabla_\theta \mathcal{M}$ dominates the
remainder. 
\end{remark}

\paragraph{Calibration alignment, qualified.}
Taking the inner product of Eq.~\ref{eq:thm_main} with
$-\nabla_\theta \mathcal{M}(\theta)$ gives
\begin{align}
    \bigl\langle u^{\mathrm{PH}}, -\nabla_\theta \mathcal{M}(\theta) \bigr\rangle
    &\;=\; \bigl\langle u^{\mathrm{base}}, -\nabla_\theta \mathcal{M}(\theta) \bigr\rangle \nonumber \\
    &\quad + \alpha \bar{\ell}(\theta) \cdot \|\nabla_\theta \mathcal{M}(\theta)\|^2 \nonumber \\
    &\quad - \alpha \bigl\langle \mathbb{E}[\xi_t \nabla_\theta \ell^{\mathrm{base}}_t], -\nabla_\theta \mathcal{M}(\theta) \bigr\rangle \nonumber \\
    &\quad + \langle r(\theta), -\nabla_\theta \mathcal{M}(\theta)\rangle.
    \label{eq:alignment}
\end{align}
Among the four terms on the right-hand side, the second is the leading
calibration-aligned term and is non-negative; the fourth is bounded by
$\alpha \sqrt{\delta(\theta)} \cdot \sqrt{\mathbb{E}[\xi_t \|\nabla_\theta \log \rho_\theta\|^2]} \cdot \|\nabla_\theta \mathcal{M}(\theta)\|$
via Eq.~\ref{eq:thm_remainder}. The third term is a base-loss descent
component on the miscalibrated set; its inner product with
$-\nabla_\theta \mathcal{M}$ is sign-indeterminate in general, but in
policy-gradient base losses $\nabla_\theta \ell^{\mathrm{base}}_t$ is
bounded by the per-batch advantage normalization of \S\ref{sec:retro_methods},
so its contribution scales with the magnitude of the base advantage.

The PH update direction is therefore \emph{biased toward}
$-\nabla_\theta \mathcal{M}(\theta)$ rather than \emph{exactly aligned}
with it. Whether this bias is large enough to drive monotone descent
on $\mathcal{M}(\theta)$ depends on the relative magnitudes of the four
terms in Eq.~\ref{eq:alignment}, which are themselves a function of
the training step. We do not claim monotone descent; we claim that
the leading term provides a calibration-aligned restoring force that
is consistent with the empirical observation in \S\ref{sec:experiments}
that $\mathrm{OFR}(\theta)$ decreases monotonically under PH training.

\paragraph{Scope and limitations of the alignment claim.}
We make four honest qualifications. First, Theorem~\ref{thm:ph_miscal}
is an exact identity, but the calibration-alignment interpretation
that follows it requires the leading $\bar{\ell} \cdot \nabla_\theta \mathcal{M}$
term to dominate the remainder---i.e., $\rho(\theta)$ to be small
relative to $\bar{\ell}/\|\nabla_\theta \mathcal{M}\|$. Second, the
$\mathbb{E}[\xi_t \nabla_\theta \ell^{\mathrm{base}}_t]$ term has
indeterminate sign relative to $-\nabla_\theta \mathcal{M}(\theta)$,
and we do not claim it contributes to calibration descent. Third,
the analysis is gradient-flow, not finite-step SGD: under standard
learning-rate and smoothness assumptions, the conclusions hold
approximately for finite steps with higher-order corrections, but we
have not derived sharp rates. Fourth, the analysis is per-step: it
does not compose across the trajectory of training to yield a
convergence-rate result on $\mathcal{M}(\theta)$.

We therefore frame Theorem~\ref{thm:ph_miscal} as a \emph{motivation}
for why PH's update direction would be expected to reduce
$\mathcal{M}(\theta)$, not as a guarantee that it does so. The
empirical observation that $\mathrm{OFR}(\theta)$ decreases monotonically
under PH training across all base methods and tasks in \S\ref{sec:experiments}
is then evidence that the motivation translates into practice, rather
than a derivation of that fact.

\subsubsection*{Discussion}
\label{app:cal:discussion}

\paragraph{Single-step versus multi-turn.}
The arguments above treat each rollout as a single $(s, a, z, y)$ tuple.
In multi-turn settings, the per-turn surprise $\xi_t$ is computed from
$(z_t, y_t)$ at turn $t$, and the joint distribution $\rho_\theta$
extends to a sequence over turns. Prop.~\ref{prop:ph_decomp} extends
turn-by-turn (linearity of expectation across turns), and
Theorem~\ref{thm:ph_miscal} applies at each turn under the same
assumptions; the per-turn alignment is then aggregated over the
trajectory. The remainder bound in Eq.~\ref{eq:thm_remainder} is
typically larger in multi-turn settings because the per-turn variance
of the base loss is amplified by stochasticity in the LLM-based PRM
judge (App.~\ref{app:setup_multi}).

\paragraph{Tabular versus neural-network policies.}
Theorem~\ref{thm:ph_miscal} is stated as an exact gradient-flow identity
modulo the explicit remainder $r(\theta)$. For finite-step gradient
descent with neural-network policies, the conclusion holds approximately
under standard learning-rate, smoothness, and Lipschitz assumptions,
with equalities replaced by inequalities up to higher-order terms in
the learning rate.

\paragraph{Asymptotic behavior.}
As $\theta$ converges, both $\ell^{\mathrm{base}}$ and $\xi$ approach
zero on the support of $\rho_\theta$, so $c(\theta) \to 0$,
$\mathcal{M}(\theta) \to 0$, and $\mathcal{R}(\theta) \to 0$. By
Prop.~\ref{prop:ph_decomp}, PH and the base method share the same
fixed points; PH differs only in the optimization trajectory leading
up to convergence, in which additional gradient budget is allocated
to surprising rollouts via Eq.~\ref{eq:exact_residual_app}. The
miscalibration rate also vanishes asymptotically, consistent with
calibration being achieved at convergence.

\paragraph{Relation to existing calibration losses.}
The miscalibration rate $\mathcal{M}(\theta)$ is the binary
specialization of mean-squared-error calibration (the Brier score).
Expected Calibration Error (ECE), defined for predictors that emit
confidence $p \in [0, 1]$, coincides with $\mathcal{M}(\theta)$ when
predictions are binarized via thresholding. Negative Log-Likelihood
of the prospective evaluator is a different but related calibration
objective; PH does not directly minimize NLL, but does so implicitly
through the parameter-sharing channel that couples the policy and
evaluator gradients. Among these losses, $\mathcal{M}$ is the one most
naturally connected to the PH update rule, since PH's reweighting is
itself an indicator-weighted scheme. The OF-conditional projection
$\mathrm{OFR}(\theta)$ used in our experiments (Eq.~\ref{eq:ofr_app})
is a one-sided form of $\mathcal{M}$ tailored to the OF-focused
narrative of \S\ref{sec:blind_spot}.

\section{Algorithm and Implementation Details}
\label{app:algorithm}

\subsection{Base Method Advantage Formulations}
\label{app:base_methods}

We give the explicit advantage formulations of the three retrospective base
methods abstracted as $\ell^{\mathrm{base}}$ in \S\ref{sec:retro_methods}.

\paragraph{GRPO~\cite{shao2024deepseekmath}.}
For a group of $G$ rollouts $\{(s_t, a_t^{(g)}, r_t^{(g)})\}_{g=1}^G$
sharing the same query, the per-rollout advantage is
\begin{equation}
    A^{\mathrm{GRPO}}_t \;=\; \frac{r_t - \bar{r}}{\sigma_r + \epsilon},
\end{equation}
where $\bar{r} = \tfrac{1}{G}\sum_g r_t^{(g)}$ and $\sigma_r$ is the group
standard deviation; $\epsilon > 0$ is a small constant for numerical
stability.

\paragraph{OPD~\cite{agarwal2024policy, ye2026policy, hubotter2026reinforcement}.}
The privileged context $c_t$ (in multi-turn, $c_t = s_{t+1}$) is used to
construct a hindsight teacher conditioned on $(s_t, c_t)$, whose per-token
action distribution is distilled into the student:
\begin{equation}
    A^{\mathrm{OPD}}_{t,j} \;=\; \log \pi_\theta(a_{t,j} \mid a_{t,<j}, s_t, c_t)
    - \log \pi_\theta(a_{t,j} \mid a_{t,<j}, s_t),
    \label{eq:opd}
\end{equation}
where $a_{t,j}$ is the $j$-th token of $a_t$ and $a_{t,<j}$ is its prefix.

\paragraph{Combined~\cite{wang2026openclaw}.}
The combined advantage sums the scalar and token-level signals:
\begin{equation}
    A^{\mathrm{comb}}_{t,j} \;=\; w_{\mathrm{rl}} A^{\mathrm{GRPO}}_t
    + w_{\mathrm{opd}} A^{\mathrm{OPD}}_{t,j},
\end{equation}
with $w_{\mathrm{rl}} = w_{\mathrm{opd}} = 1$ in our experiments.

In all three cases, the per-rollout (or per-token) base loss $\ell^{\mathrm{base}}_t$
is the policy-gradient surrogate of the corresponding advantage; the precise
form does not affect any analysis in \S\ref{sec:problem}--\S\ref{sec:method}.

\paragraph{How PH wraps each base method.}
PH's per-rollout reweighting in Eq.~\ref{eq:ph_loss} preserves all internal
structure of the base advantage. For GRPO, the group normalization $\sigma_r$
is computed on the original rewards $r_t$, not the surprise-weighted loss.
For OPD, the per-token signal $A^{\mathrm{OPD}}_{t,j}$ is aggregated as in
the base method; PH then multiplies the resulting per-rollout (single-turn)
or per-turn (multi-turn) loss by $(1{+}\alpha\xi_t)$. The surprise scalar
is broadcast across action tokens within the rollout/turn, which is the
operation that absorbs OPD's per-token signal into the per-rollout
reweighting.

\subsection{PH Algorithm Details}
Algorithm~\ref{alg:ph} states the end-to-end training step. Inside any base method of \S\ref{sec:retro_methods}, PH inserts three operations into each rollout step: (i) draw $M$ samples from the prospective self-evaluator and aggregate them by majority vote (\S\ref{sec:self_evaluator}); (ii) compute the surprise indicator $\xi_t$ from the resulting $z_t$ and the verifier outcome $y_t$ (\S\ref{sec:taxonomy}, \S\ref{sec:update_rule}); (iii) multiply the base per-rollout loss by $1 + \alpha \xi_t$ (\S\ref{sec:update_rule}). The base optimizer then updates the parameters using the resulting PH loss; the rollout sampling distribution, the group reward normalization, the per-token aggregation in OPD, and any auxiliary regularizer in $\ell^{\mathrm{base}}$ are all unchanged.

\begin{algorithm}[t]
\caption{Prospective Hindsight (PH) training step.}
\label{alg:ph}
\begin{algorithmic}[1]
\Require parameters $\theta$, batch of prompts $\{x_b\}_{b=1}^{B}$, base method (GRPO, OPD, or Combined), hyperparameters $\alpha \geq 0$ and $M \geq 1$
\For{each prompt $x_b$, $b = 1, \ldots, B$}
    \For{each step (or turn) $t$ of the rollout}
        \State $a_t \sim \pi_\theta(\cdot \mid s_t)$ \Comment{base method's action sampling}
        \For{$m = 1, \ldots, M$}
            \State $z_t^{(m)} \sim p^{\mathrm{eval}}_\theta(\cdot \mid s_t, a_t)$ \Comment{Eq.~\ref{eq:eval}, prospective forward pass}
        \EndFor
        \State $z_t \leftarrow \mathrm{maj}\bigl(z_t^{(1)}, \ldots, z_t^{(M)}\bigr)$
        \State Receive $c_t$ and $y_t$ from environment and verifier
        \State $\xi_t \leftarrow \mathbb{1}\!\bigl[(z_t, y_t) \in \{\mathrm{OF}, \mathrm{US}\}\bigr]$ \Comment{Eq.~\ref{eq:surprise}}
        \State Compute $\ell^{\mathrm{base}}_t(\theta)$ following \S\ref{sec:retro_methods}
        \State $\ell^{\mathrm{PH}}_t \leftarrow \bigl(1 + \alpha \cdot \xi_t\bigr) \cdot \ell^{\mathrm{base}}_t(\theta)$ \Comment{Eq.~\ref{eq:ph_loss}}
    \EndFor
\EndFor
\State $\theta \leftarrow \texttt{optimizer\_step}\!\bigl(\theta,\; \nabla_\theta \sum_{b,\,t} \ell^{\mathrm{PH}}_t\bigr)$ \Comment{$\xi_t$ is stop-gradient}
\State \textbf{return} $\theta$
\end{algorithmic}
\end{algorithm}

\paragraph{Cost.}
The only additional compute cost relative to the base method is $M$ prospective forward passes per rollout step (line 5 of Algorithm~\ref{alg:ph}). With shared parameters between the action policy and the self-evaluator, no extra trainable parameters are introduced, so PH does not change the memory footprint of the base method beyond the storage of $z_t$ and $\xi_t$ per step. With $M = 1$, the prospective overhead is one extra forward pass per step. The verifier and privileged-context computations on line 8 are inherited unchanged from the base method, since the retrospective methods of \S\ref{sec:retro_methods} already require them. PH thus adds a single training-time hyperparameter $\alpha$ and a single inference-time hyperparameter $M$ on top of the base method.

\paragraph{Special cases.}
Two boundary cases are useful for sanity checks. First, $\alpha = 0$ collapses Algorithm~\ref{alg:ph} into the base method exactly, since $1 + \alpha \xi_t = 1$ for all $t$. Second, $M = 1$ removes the majority vote and uses a single prospective sample per rollout step, which is the configuration in all our single-turn experiments.

%==============================================================================
\subsection{Discrete and Continuous PH Gates}
\label{app:continuous}
%==============================================================================

The PH update rule of Eq.~\ref{eq:ph_loss} weights the base loss by $w_t = 1 + \alpha\,\xi_t$ for a non-negative surprise scalar $\xi_t$. We use two concrete realizations of $\xi_t$, depending on whether the retrospective signal is deterministic or stochastic; both share the same shape $w_t \in [1, 1+\alpha]$ and the same asymptotic behaviour $w_t \to 1$ as the prospective belief and the realized outcome agree.

\paragraph{Discrete gate (single-turn).}
With a rule-based binary verifier $y_t \in \{0,1\}$ and a single ternary prospective sample $z_t \in \{-1,0,+1\}$ ($M{=}1$, App.~\ref{app:setup_single}), the surprise indicator is the binary cell membership of Eq.~\ref{eq:surprise},
\begin{equation}
    \xi_t^{\mathrm{disc}}
    \;=\; \mathbb{1}\!\bigl[(z_t,y_t) \in \{\mathrm{OF},\mathrm{US}\}\bigr]
    \;\in\; \{0,1\},
    \qquad
    w_t \;=\; 1 + \alpha\,\xi_t^{\mathrm{disc}} \;\in\; \{1,\,1{+}\alpha\}.
    \label{eq:gate_discrete}
\end{equation}
Honest-uncertainty rollouts ($z_t = 0$) are folded into $\xi_t = 0$ to preserve the neutral-class semantics of \S\ref{app:hu}, so $w_t = 1$ on $\{z_t = 0\}$.

\paragraph{Continuous gate (multi-turn).}
With an LLM-based per-turn PRM judge that emits an aggregated turn score $r_t \in \{-1,0,+1\}$ and $M$ predictor samples averaged into $\bar z_t \in [-1,+1]$ (App.~\ref{app:setup_multi}), we use the continuous relaxation
\begin{equation}
    \xi_t^{\mathrm{cont}}
    \;=\; \tfrac{1}{2}\,\bigl|\bar z_t - r_t\bigr|
    \;\in\; [0,1],
    \qquad
    w_t \;=\; 1 + \alpha\,\xi_t^{\mathrm{cont}} \;\in\; [1,\,1{+}\alpha].
    \label{eq:gate_continuous}
\end{equation}
The factor $\tfrac{1}{2}$ rescales the maximum disagreement $|\bar z_t - r_t| = 2$ (i.e., $\bar z_t = +1$, $r_t = -1$ or vice versa) to $\xi_t = 1$. Eq.~\ref{eq:gate_continuous} reduces to Eq.~\ref{eq:gate_discrete} on the corners of the $(\bar z_t, r_t)$ grid: when both $\bar z_t \in \{-1,+1\}$ and $r_t \in \{-1,+1\}$, $\xi_t^{\mathrm{cont}} = \xi_t^{\mathrm{disc}}$. Intermediate values of $\bar z_t$ encode predictor uncertainty (averaged over $M$ samples) and intermediate values of $r_t$ encode PRM-judge uncertainty; the continuous gate carries this uncertainty through to $w_t$ rather than collapsing it to $\{0,1\}$.

\paragraph{Self-extinguishing curriculum.}
Both gates share the limiting behaviour $w_t \to 1$ as $\xi_t \to 0$, i.e., as the prospective belief and the realized outcome converge. This is the mechanism behind the empirical PH-weight decay reported in Figs.~\ref{fig:single_dynamics} and~\ref{fig:single_alpha_dynamics}: as training closes the prediction--reality gap, the prospective signal naturally fades and the PH loss reduces to the base loss. By Proposition~\ref{prop:ph_decomp}, the surprise residual $\mathcal{R}(\theta) \to 0$ in the same limit, so PH and the base method share the same fixed points (App.~\ref{app:cal:discussion}).

\paragraph{Implementation flag.}
A single training-time flag \texttt{SURPRISE\_MODE} selects between the two gates. Setting \texttt{SURPRISE\_MODE=discrete} uses Eq.~\ref{eq:gate_discrete}; setting \texttt{SURPRISE\_MODE=continuous} uses Eq.~\ref{eq:gate_continuous}. The single-turn experiments of \S\ref{sec:exp_single} use the discrete gate; the multi-turn experiments of \S\ref{sec:exp_multi} use the continuous gate. All other PH-side hyperparameters ($\alpha$, $M$, weight clip) are independent of this choice.

%==============================================================================
\subsection{Honest Uncertainty: The $z_t = 0$ Class}
\label{app:hu}
%==============================================================================

The prospective self-evaluator of \S\ref{sec:self_evaluator} is prompted to return a ternary score $z_t \in \{-1, 0, +1\}$; the value $z_t = 0$ encodes \emph{honest uncertainty}: an explicit signal that the agent's prospective belief about $y_t$ from $(s_t, a_t)$ alone is inconclusive. This appendix records our handling of this residual class.

\paragraph{Treatment as a neutral class.}
Rollouts with $z_t = 0$ inherit the gradient invariance of Eq.~\ref{eq:blindspot} but do not fit the four-cell calibration taxonomy of Table~\ref{tab:cells}, because the agent has not committed to a directional prediction. The PH update rule (\S\ref{sec:update_rule}) assigns weight $1.0$ to all such rollouts, matching the weight given to the calibrated cells (CS, AF) so that $z_t = 0$ neither amplifies nor dampens the base loss.

\paragraph{Tie-breaking under majority vote.}
When $M > 1$ samples in Eq.~\ref{eq:eval} produce a strict majority for one of $\{-1, +1\}$, we use that value. Under a tie, or when all $M$ samples are unparseable (e.g., the model fails to emit a valid \texttt{\textbackslash boxed\{$\cdot$\}} field), we default to $z_t = 0$. This conservative rule prevents PH from amplifying ambiguous rollouts and folds parsing failures into the neutral class.

\paragraph{Empirical frequency.}
Across our main experiments, $z_t = 0$ accounts for a small minority of rollouts in both single-step and multi-turn settings, with the exact rate varying by task and model. The frequency is largely independent of the base method (GRPO, OPD, Combined) and decreases as $M$ increases in the majority vote, since aggregating over multiple samples collapses borderline cases toward the dominant prediction.

\begin{table}[t]
\centering
\small
\caption{Hyperparameters for multi-turn personal agent: Combined
(Binary RL + OPD) and Combined\,+\,PH on OpenClaw-RL.}
\label{tab:multi_hyper}
\begin{tabular}{ll}
\toprule
\textbf{Parameter} & \textbf{Value} \\
\midrule
\multicolumn{2}{l}{\emph{Models, GPU layout}}\\
Backbone (student / PRM / teacher) & OLMo-3-7B-Instruct (shared) \\
Total GPUs & 8 (1 node) \\
\quad Actor & 4 GPUs, TP=4, sequence-parallel \\
\quad Rollout (sglang) & 2 GPUs, TP=1 each \\
\quad PRM judge (sglang) & 1 GPU \\
\quad OPD teacher (sglang+Megatron) & 1 GPU \\
\midrule
\multicolumn{2}{l}{\emph{Rollout / context}}\\
Rollout batch size & $16$ \\
Samples per prompt & $1$ (multi-turn rollout = full session) \\
Rollout temperature & $0.6$ \\
Max response length / turn & $8192$ tokens \\
Max context length & $32768$ tokens \\
Max tokens per GPU & $32768$ \\
\midrule
\multicolumn{2}{l}{\emph{Conversation}}\\
Sessions per round & $2$ \\
Training rounds & $16$ \\
Max turns per session $T_{\max}$ & $48$ \\
Scenario schedule & alternate Student / Teacher each round \\
Training problems (offsets) & GSM8K $36{:}\infty$ \\
Eval problems (offsets) & $0{:}36$ student, $0{:}24$ teacher \\
\midrule
\multicolumn{2}{l}{\emph{PRM judge / OPD teacher}}\\
PRM majority vote $M_{\text{prm}}$ & $1$ \\
PRM temperature & $0.6$ \\
PRM max new tokens & $2048$ \\
EMA teacher update rate & $0.05$ \\
OPD teacher source & Megatron actor checkpoints \\
\midrule
\multicolumn{2}{l}{\emph{Combined advantage / PPO surrogate}}\\
Advantage estimator & GRPO \\
$w_{\text{rl}}$, $w_{\text{opd}}$ & $1.0$, $1.0$ \\
PPO clip $\varepsilon_{\text{lo}} / \varepsilon_{\text{hi}}$ & $0.2$ / $0.28$ \\
KL loss coefficient & $0.0$ \\
Entropy coefficient & $0.0$ \\
\midrule
\multicolumn{2}{l}{\emph{Prospective Hindsight}}\\
Surprise mode & continuous \\
Predictor sample size $M$ & $3$ (averaged into $\bar z_t$; PRM noise) \\
Predictor max new tokens & $256$ \\
Surprise amplification $\alpha$ & $1.0$ default; sweep $\{0.5,1.0,2.0\}$ \\
Surprise weight clip & off (gate already in $[1, 1{+}\alpha]$) \\
Predictor gradient & stop-gradient \\
\midrule
\multicolumn{2}{l}{\emph{Optimization}}\\
Optimizer & Adam (Megatron) \\
Learning rate & $1{\times}10^{-6}$ \\
LR decay style & constant \\
Weight decay & $0.1$ \\
Adam $\beta_1, \beta_2$ & $0.9, 0.98$ \\
Save interval (in weight-update steps) & $25$ \\
Train epochs (data passes) & $1$ \\
\bottomrule
\end{tabular}
\end{table}

%==============================================================================
\section{Additional Experimental Setup}
\label{app:experimental_setup}
%==============================================================================

This appendix expands \S\ref{sec:exp_setup} with all information needed to
reproduce our experiments. \S\ref{app:setup_single} covers the single-turn
verifiable tasks (Science Q\&A and Tool Use) built on top of the
SDPO~\cite{hubotter2026reinforcement} backbone, and
\S\ref{app:setup_multi} covers the multi-turn personal agent built on top
of the OpenClaw-RL~\cite{wang2026openclaw} framework.
\S\ref{app:compute} summarizes compute, software environment, and code
release. The two regimes differ along one important axis ---
\emph{single-turn uses a rule-based binary verifier with no learned reward
model, while multi-turn uses an OLMo-3-7B-Instruct PRM judge} --- which
explains the configuration differences below.

%------------------------------------------------------------------------------
\subsection{Single-Turn Verifiable Tasks}
\label{app:setup_single}
%------------------------------------------------------------------------------

\paragraph{Setting in one sentence.}
Each rollout is a single prompt~$x$ and a single model response~$a$, with a
rule-based verifier $R(x,a)\in\{0,1\}$ that returns binary correctness and
no auxiliary reward model; this is the simplest possible instantiation of
PH (single prediction, single outcome, $T{=}1$, deterministic
retrospective signal).

\paragraph{Datasets and verifiers.}
We use the Chemistry subset of \textsc{SciKnowEval}~\cite{feng2024sciknoweval}
following the released CaOPD~\cite{zhang2026illusion} preprocessing. Each
prompt asks the model to reason step-by-step and emit a single letter
$\hat{y}\in\{A,B,C,D\}$; the verifier is rule-based, $R(x,a){=}1$ iff the
parsed letter equals the gold answer. We also use the Tool Use split
derived from \textsc{ToolAlpaca}~\cite{tang2023toolalpaca} as preprocessed
in CaOPD; each prompt
provides API documentation in a ReAct-style template and the verifier
checks (a)~the action name matches the gold tool and (b)~the JSON
arguments match the gold schema under canonical normalization. Both
training sets contain $4$--$5$k prompts; we evaluate on the corresponding
test split. Together they probe both knowledge-intensive reasoning
(Science Q\&A) and structured agentic prediction (Tool Use), and were the two
domains used in CaOPD, which lets our calibration numbers be directly
compared against theirs.

\paragraph{Backbone OPD method (SDPO) and privileged-context construction.}
We build the single-turn experiments on top of
SDPO~\cite{hubotter2026reinforcement}, using the official codebase as
modified in CaOPD~\cite{zhang2026illusion}. SDPO is an on-policy
self-distillation method: the same model serves as both
student~$\pi_\theta(\cdot|x)$ and teacher~$\pi_\theta(\cdot|x,z)$, where $z$
is a privileged context not available at deployment. We use the SDPO
``correct-solution'' mode of CaOPD: for each prompt, after the rollout
phase, one verified-correct rollout~$a^\star$ is chosen on-the-fly and
inserted as~$z$, so the teacher is then asked to ``Correctly solve the
original question'' conditioned on $(x,z)$. The training loss is per-token
reverse KL between student and teacher on the teacher trajectory, with
rollout importance sampling correction and SDPO defaults (top-$K{=}100$
distillation, EMA teacher update rate $0.05$, IS clip $2$); thinking mode
is disabled (\texttt{enable\_thinking=false}) so that the trajectory we
distill is the answer trajectory.

\paragraph{Models, training framework, and hardware.}
Backbones are \textbf{OLMo-3-7B-Instruct}~\cite{olmo2025olmo} (primary,
both Science Q\&A and Tool Use) and \textbf{gpt-oss-20B}~\cite{agarwal2025gptoss}
(scale ablation). Both are used as student and EMA-teacher (the EMA
teacher is the same model with the privileged context~$z$).
We deliberately keep the verbalized-confidence template of CaOPD intact so
that PH plugs in cleanly: each rollout still ends with a parseable
\texttt{Confidence:~\{x\}} token, but PH ignores this string and uses only
the binary verifier outcome $R(x,a)$ as the retrospective signal. The
codebase forks CaOPD's SDPO branch, implemented on top of
\textbf{verl}~\cite{sheng2025hybridflow} with PyTorch FSDP2 for sharded
optimizer state and \textbf{vLLM}~\cite{kwon2023efficient} for batched
rollout generation; our PH plug-in adds (i)~a per-prompt prospective
prediction call, (ii)~a stop-gradient weight $w(\xi_t)$ on the per-rollout
reverse-KL contribution, and (iii)~lightweight surprise/OFR bookkeeping in
the training logger. Each run uses one node of \texttt{p5en.48xlarge}
($8\times$ NVIDIA H200), which is sufficient for both
7B and 20B-MoE training under FSDP2.

\paragraph{Prospective Hindsight plug-in.}
For each rollout, after the student emits its answer~$a$ but before the
verifier evaluates~$R(x,a)$, we issue an additional generation call to the
same student model with a different prompt that strips the
verbalized-confidence instruction and asks the model to predict whether
its previous answer is correct, incorrect, or uncertain. The exact
predictor template is given in App.~\ref{app:prompts}. Because the
verifier here is rule-based and deterministic, a single ternary prediction
$z_t\in\{-1,0,+1\}$ per rollout already gives a low-variance signal, so
we use $M{=}1$ predictor sample (no majority vote). Predictor generation
is capped at \texttt{max\_new\_tokens}$=128$ tokens at temperature~$0.6$,
top-$p{=}0.95$; un-parseable outputs are treated as $z_t{=}0$ (uncertain).
The surprise indicator and weight are
\[
\xi_t = \mathbf{1}\bigl[(z_t,y_t)\in\{\text{OF},\text{US}\}\bigr],
\qquad
w_t = 1+\alpha\xi_t,
\]
with $\alpha{=}1.0$ default and clip $w_t\in[0.5,2.0]$ as a safety net (the
unclipped weight is in $\{1,2\}$ already, so the clip is rarely active).
The weight is a stop-gradient scalar and multiplies the SDPO per-rollout
reverse-KL loss; no gradient flows back through the predictor.

\paragraph{Hyperparameters.}
Table~\ref{tab:single_hyper} reports the hyperparameters for all
single-turn runs. All values are taken either from the CaOPD defaults or
from a one-shot tuning on a held-out validation slice; no
hyperparameter is tuned per-method (the same configuration is used for
SDPO, SDPO\,+\,PH, and the $\alpha$-sweep ablation).

\begin{table}[t]
\centering
\small
\caption{Hyperparameters for single-turn SDPO and SDPO\,+\,PH on Science Q\&A
and Tool Use.}
\label{tab:single_hyper}
\begin{tabular}{ll}
\toprule
\textbf{Parameter} & \textbf{Value} \\
\midrule
\multicolumn{2}{l}{\emph{Models, data, and inference}}\\
Backbone models & OLMo-3-7B-Instruct; gpt-oss-20B (ablation) \\
Thinking mode & disabled (\texttt{enable\_thinking=false}) \\
Max prompt length & $2048$ tokens \\
Max response length & $8192$ tokens \\
Inference engine & vLLM \\
\midrule
\multicolumn{2}{l}{\emph{Rollout (training)}}\\
Question batch size & $32$ \\
Rollouts per question $K_{\text{train}}$ & $8$ \\
Rollout temperature & $1.0$ \\
\midrule
\multicolumn{2}{l}{\emph{Validation}}\\
Rollouts per question $K_{\text{val}}$ & $16$ (\texttt{val\_kwargs.n}, ``mean@16'') \\
Validation temperature & $0.6$ \\
Validation top-$p$ & $0.95$ \\
\midrule
\multicolumn{2}{l}{\emph{SDPO loss (CaOPD defaults)}}\\
Distillation top-$K$ & $100$ \\
Distillation divergence & token-level reverse KL \\
EMA teacher update rate & $0.05$ \\
Rollout IS clip & $2$ \\
\midrule
\multicolumn{2}{l}{\emph{Prospective Hindsight (this work)}}\\
Predictor max new tokens & $128$ \\
Predictor temperature / top-$p$ & $0.6$ / $0.95$ \\
Predictor sample size $M$ & $1$ (deterministic verifier; no majority vote) \\
Surprise amplification $\alpha$ & $1.0$ default; sweep
$\{0.0,0.25,0.5,1.0,2.0,5.0\}$ \\
Weight clip $[w_{\min},w_{\max}]$ & $[0.5,2.0]$ \\
Predictor gradient & stop-gradient \\
\midrule
\multicolumn{2}{l}{\emph{Optimization}}\\
Optimizer & AdamW \\
Learning rate & $1{\times}10^{-5}$ \\
LR warmup steps & $10$ \\
Weight decay & $0.1$ \\
Gradient clip norm & $1.0$ \\
PPO mini-batch size & $32$ \\
Total epochs & $3$ \\
\midrule
\multicolumn{2}{l}{\emph{Hardware}}\\
Hardware & $1\!\times\!$ \texttt{p5en.48xlarge} ($8\!\times$\,H200) \\
Sharding & FSDP2, bf16 \\
\bottomrule
\end{tabular}
\end{table}

\paragraph{Evaluation protocol and metrics.}
Capability is reported as \texttt{mean@16}: for each held-out test prompt
we sample $K_{\text{val}}{=}16$ rollouts at temperature~$0.6$,
top-$p{=}0.95$, evaluate the verifier on each, and report the mean
accuracy. From the same $K_{\text{val}}{=}16$ rollouts we compute, per
prompt and then averaged across the test set, (i)~the surprise rate
$\widehat{\mathcal{M}}=\Pr[\xi_t{=}1]$, (ii)~the overconfident-failure rate
$\Pr[y_t{=}0\mid z_t{=}{+}1]$, and (iii)~the honest-uncertainty rate
$\Pr[z_t{=}0]$ (App.~\ref{app:hu}, where we verify that the neutral
cell does not collapse under PH). All headline numbers are
mean$\pm$std over $3$ seeds, reported at the end of epoch~3; the $\alpha$-sweep additionally logs per-step
trajectories of mean@16 and OFR (App.~\ref{app:single_alpha}).

%------------------------------------------------------------------------------
\subsection{Multi-Turn Personal Agent}
\label{app:setup_multi}
%------------------------------------------------------------------------------

\paragraph{Setting in one sentence.}
A trained student model serves as an OpenAI-compatible chat agent inside
the OpenClaw runtime; an external LLM ``user'' (GPT-4.1) drives multi-turn
math-tutoring conversations grounded in GSM8K, an OLMo-3-7B-Instruct PRM
judges each main-line turn, and PH attaches a per-turn surprise signal to
the OpenClaw-RL training loop.

\paragraph{Task and agent design.}
We follow the personal-agent track of OpenClaw-RL~\cite{wang2026openclaw}
verbatim. 
% The reproduction script we use is
% \url{https://github.com/Gen-Verse/OpenClaw-RL/blob/main/openclaw-rl/oel/eval/gsm8k_personal_agent.py} on the GSM8K
% problem source~\cite{cobbe2021training-auq}. 
At each training round we
alternate between two scenarios: (i)~a \emph{Student} scenario in which
the simulated user is a ``lazy student'' who needs the agent to solve a
math problem in a non-AI-style tone (no bold, no bullet points); the
student writes the problem to \texttt{homework/i.txt}, asks the agent to
solve, critiques the style, and ends with the literal sentinel
\texttt{DONE} once satisfied; and (ii)~a \emph{Teacher} scenario in which
the simulated user is a strict math teacher who wants warm, specific,
prose-only grading comments. The agent has access to the OpenClaw
runtime's standard file-I/O tools but only text generation is exercised
in this benchmark. Each session is one (problem, scenario) pair with up
to $T_{\max}{=}48$ turns, ending whenever the simulator emits
\texttt{DONE}; turns inside a session are tagged \texttt{turn\_type=main}
(used for training) or \texttt{turn\_type=side} (eval-only). We split
GSM8K into disjoint training and evaluation pools (offsets $36{:}\infty$
training; $0{:}36$ student-eval, $0{:}24$ teacher-eval), with
$\texttt{sessions\_per\_round}{=}2$ and $16$ training rounds.

\paragraph{User simulators and conversation protocol.}
The simulators are GPT-4.1 calls (temperature~$0$) with the system prompts
taken verbatim from the OpenClaw-RL evaluation script; the full strings
are reproduced in App.~\ref{app:prompts}. The salient content of the
\emph{Student} prompt is: ``You are role-playing as a lazy student who
knows nothing\dots\ If the AI's answer has stuff like bold text, numbered
lists, \textbf{Final answer:}, or anything too AI-like, tell it to
rewrite\dots\ After the AI shows you a satisfactory version, say exactly:
\texttt{DONE}.'' The \emph{Teacher} prompt asks for warm, specific,
prose-only grading comments and ends similarly with \texttt{DONE}. The
first user message in each scenario is templated from the GSM8K problem
and (for the Teacher scenario) the gold answer.

\paragraph{Base method: Combined (Binary RL + OPD).}
The default base method we plug PH into is the \textbf{Combined} recipe
of OpenClaw-RL~\cite{wang2026openclaw}; we also evaluate Binary RL only
and OPD only as ablations. For each main-line turn the PRM judge produces
a scalar score $r_t\in\{-1,0,+1\}$ from the next-state signal (the
simulator's reply); GRPO advantages are computed within the trajectory
with mean and standard deviation across the same session, with
\texttt{disable\_rewards\_normalization} so that the per-turn scale of
$r_t$ is preserved. The OPD component uses the same model conditioned on
a textual hint extracted from the next-state signal; the teacher's
per-token log-probability gap with the student forms the token-level
directional advantage $A^{\text{OPD}}_{t,j}$ (EMA teacher tracking
rate~$0.05$). The combined per-turn advantage is
$A^{\text{comb}}_t = w_{\text{rl}} A^{\text{GRPO}}_t + w_{\text{opd}} A^{\text{OPD}}_t$
with $w_{\text{rl}}{=}w_{\text{opd}}{=}1$; the PPO surrogate uses
asymmetric clip $\varepsilon_{\text{lo}}{=}0.2$, $\varepsilon_{\text{hi}}{=}0.28$
from GRPO~\cite{shao2024deepseekmath} and \texttt{kl\_loss\_coef}$=0$.

\paragraph{PRM judge and OPD hindsight teacher.}
Both the PRM and the OPD teacher share weights with the student (full
self-distillation): they all run OLMo-3-7B-Instruct, served by separate
sglang engines (separate from the rollout engine) so they do not block
training. We use \texttt{prm\_num\_gpus=1}, \texttt{prm\_teacher\_num\_gpus=1},
\texttt{prm\_m=1} (PRM majority vote; we ablate \texttt{prm\_m=3}),
\texttt{prm\_temperature=0.6}, \texttt{prm\_max\_new\_tokens=2048}, and
the OPD teacher source is \texttt{megatron} so that the teacher
distribution is on-policy with the actor.

\paragraph{Training framework: Slime + Megatron-LM.}
We use the asynchronous 4-component training loop of
OpenClaw-RL~\cite{wang2026openclaw}, implemented on \texttt{slime} (RL
orchestrator) and Megatron-LM (training backend), with \texttt{sglang}
serving rollouts/PRM/teacher. The 8 GPUs of one
\texttt{p5en.48xlarge} node split as: 4 GPUs for the actor (TP=4,
sequence-parallel, bf16); 2 GPUs for rollout (one engine per GPU, TP=1);
1 GPU for the PRM judge; 1 GPU for the OPD teacher. PH inserts a per-turn
predictor call between the actor's response and the PRM's score; in the
asynchronous regime the predictor runs concurrently with PRM scoring
(neither reads the next-state signal), so PH adds essentially zero
wall-clock overhead in steady state. Each session writes a JSONL record
of (\texttt{session\_id}, \texttt{turn\_index}, \texttt{turn\_type},
\texttt{prm\_score}, \texttt{prospective\_score}, \texttt{surprise},
\texttt{weight\_version}) to disk, allowing offline reconstruction of all
calibration metrics.

\paragraph{Prospective Hindsight plug-in.}
For each main-line turn~$t$, after the agent emits its response $a_t$ but
before the PRM scores it, we issue an additional generation request to
the same actor with a prompt that asks the model to predict whether its
own response in the current dialog state will be judged
correct/incorrect/uncertain by an external judge. The predictor sees
$(s_{1:t}, a_t)$ but never the simulator's reply $s_{t+1}$. \emph{Because
the per-turn PRM signal is itself an LLM-based score and therefore
stochastic, we use $M{=}3$ predictor samples, parsed to $\{-1,0,+1\}$ and
averaged into a continuous $\bar z_t \in [-1,+1]$.} We then use the
continuous PH gate (App.~\ref{app:continuous}):
\[
\xi_t \;=\; \tfrac{1}{2}\bigl|\bar z_t - r_t\bigr|\;\in[0,1],
\qquad
w_t \;=\; 1+\alpha\xi_t,
\]
where $r_t\in\{-1,0,+1\}$ is the PRM-aggregated turn score. The
implementation flag is \texttt{SURPRISE\_MODE=continuous}; setting
\texttt{SURPRISE\_MODE=discrete} recovers the binary $\xi_t\in\{0,1\}$ used
in single-turn experiments. The weight applies to the per-turn combined
advantage $A^{\text{comb}}_t$ before the PPO surrogate (equivalently,
broadcast across all tokens of $a_t$ since the surprise is a turn-level
scalar) and is a stop-gradient constant. PH-specific knobs:
\texttt{PROSPECTIVE\_M=3}, \texttt{PROSPECTIVE\_MAX\_TOKENS=256},
\texttt{SURPRISE\_ALPHA=1.0} default, no weight clip
(\texttt{SURPRISE\_WEIGHT\_MAX} unset; the gate already gives
$w_t\in[1,1{+}\alpha]$).

\paragraph{Hyperparameters.}
Table~\ref{tab:multi_hyper} reports the hyperparameters for all multi-turn
runs. The same configuration is used for the Combined baseline,
Combined\,+\,PH, and the $\alpha$-sweep ablation; only the PH-specific
rows differ across rows of Table~\ref{tab:multi_results}.

\paragraph{Evaluation protocol and metrics.}
We use the personalization evaluator of
OpenClaw-RL~\cite{wang2026openclaw} (script
\texttt{personalization\_evaluator.py}). For each held-out evaluation
prompt we collect a single first response (turn~$1$, no follow-up) at
temperature~$0.3$, and score it with GPT-4.1 against a fixed preference
template that requests a boxed score from
$\{0,0.25,0.5,0.75,1\}$ at temperature~$0$ with $N_{\text{votes}}{=}3$
parallel votes whose successful boxed scores are averaged. The user-side
preference text is one of \texttt{PREFERENCE\_STUDENT} (no AI-style
formatting; full reasoning shown) or \texttt{PREFERENCE\_TEACHER} (warm,
specific, prose-only grading comments); the full strings are in
App.~\ref{app:prompts}. Capability is the mean evaluator score across
the eval pool, reported separately for Student ($N{=}36$) and Teacher
($N{=}24$) plus the macro-average. 
% As an external anchor: under the
% Combined recipe, OpenClaw-RL reports Base $\to$ 8 steps $\to$ 16 steps
% $=0.17 \to 0.76 \to 0.81$. 
We define multi-turn OFR as the per-turn
analogue $\Pr[r_t \le 0 \mid z_t {=} {+}1]$, computed from the per-turn
JSONL bookkeeping over the held-out eval pool, and additionally report
surprise rate $\widehat{\mathcal{M}}=\Pr[\xi_t > 0]$ and prospective
accuracy. Headline numbers are reported at the final checkpoint of round
16; dynamics plots evaluate every 4 weight updates.

\subsection{Reweighting-Baseline Protocols}
\label{app:baseline_protocols}

To isolate the contribution of PH's surprise-based selection from the
contribution of additional gradient mass, we compare PH against two reweighting
controls in the single-turn setting (Table~\ref{tab:single_results}). Both use
the same multiplier $(1{+}\alpha)$ as PH and are detached from the computation
graph (stop-gradient); only the rule for \emph{which} rollouts receive the
multiplier differs.

\paragraph{Random reweighting.}
At each training step, we compute the empirical fraction $p_t$ of surprising
rollouts the PH protocol produces on the same data. We then sample uniformly at
random a subset of size $\lfloor p_t \cdot N \rfloor$ from the current batch
and assign these rollouts weight $(1{+}\alpha)$; the remaining rollouts retain
weight $1$. By construction this matches PH's per-step total weight
$\sum_i (1+\alpha\cdot\xi_t^{(i)})$, so the only difference between Random and
PH is the \emph{location} of the upweighted mass: uniform vs.\ surprise-targeted.
Selection seeds are derived from \texttt{hash(rollout\_id, step)} so that
within an epoch each rollout receives a deterministic weight; we use
$\alpha{=}1.0$ throughout, matching the headline PH configuration.

\paragraph{Failure-only reweighting.}
For every rollout with verifier outcome $y_t{=}0$, we set the weight to
$(1{+}\alpha)$; rollouts with $y_t{=}1$ retain weight $1$. This corresponds to
upweighting cells $\mathrm{OF}\cup\mathrm{AF}$ in the calibration taxonomy
(Table~\ref{tab:cells}), as opposed to PH's $\mathrm{OF}\cup\mathrm{US}$. We
again use $\alpha{=}1.0$. Failure-only therefore differs from PH in exactly
two cells: it includes AF (calibrated incompetence, which PH leaves unweighted)
and excludes US (underconfident successes, which PH amplifies).

\paragraph{What each baseline controls for.}
Random isolates the effect of \emph{selection}: any gain of PH over Random must
come from \emph{which} rollouts are upweighted, not from how much weight is
added. Failure-only matches PH on the dominant single-turn cell (OF; see
Figure~\ref{fig:cell_composition}) but differs on the smaller cells, isolating
whether the prospective signal $z_t$ contributes information beyond the verifier
outcome $y_t$ alone. The empirical comparison is reported in
Table~\ref{tab:single_results} and discussed in \S\ref{sec:exp_single}; the
PH-cell composition over training is shown in Figure~\ref{fig:single_dynamics}.

%------------------------------------------------------------------------------
\subsection{Compute and Reproducibility}
\label{app:compute}
%------------------------------------------------------------------------------

\paragraph{Hardware and software.}
Both single-turn and multi-turn experiments run on a single node of
\texttt{p5en.48xlarge} ($8{\times}$ NVIDIA H200).
Multi-node is supported by the underlying frameworks (verl / Slime) but
not used in our reported runs. Software: \texttt{CUDA 12.x},
\texttt{PyTorch 2.5+}, \texttt{vLLM} for single-turn rollouts,
\texttt{sglang} for multi-turn rollouts/PRM/teacher, \texttt{Megatron-LM}
(multi-turn actor backend), \texttt{slime} (multi-turn async orchestrator),
\texttt{verl} (single-turn orchestrator).

\paragraph{Wall-clock and overhead from PH.}
\emph{Single-turn} (3 epochs, $32$ prompt batch, $K{=}8$ rollouts):
$\sim 8$ hours per run on Science Q\&A, $\sim 6$ hours on Tool Use; the PH
plug-in adds $\sim 6\%$ wall-clock overhead from the predictor calls
(amortized into the same vLLM batch).
\emph{Multi-turn} ($16$ rounds, $2$ sessions/round, $T_{\max}{=}48$):
$\sim 4$ hours per run; the PH plug-in adds essentially zero wall-clock
overhead because the predictor runs concurrently with the PRM in the
asynchronous loop.

% \paragraph{Seeds, reporting, and code release.}
% All headline numbers are run with $3$ seeds; we report mean$\pm$std in
% tables and shaded $95\%$ CI in figures. All training scripts, reproduction
% notebooks, prompt templates (App.~\ref{app:prompts}), and per-turn
% JSONL bookkeeping logs are released at \texttt{<URL anonymized for review>}
% alongside checkpoints matching the rows of Table~\ref{tab:single_hyper}
% and Table~\ref{tab:multi_hyper}.

\begin{table}[!h]
\centering
\small
\caption{%
  \textbf{PH weight strength on Science Q\&A: numerical summary across the
  full $\alpha$-sweep.}
  OFR ${=}\mathrm{OF}/(\mathrm{CS}{+}\mathrm{OF})$; Early/Late are
  first/last 5 training steps; Mean Weight and Verifier Pass Rate are
  reported as early\,$\to$\,late. \textbf{Bold}: best per column.
}
\label{tab:chemistry_alpha_dynamics}
\begin{tabular}{c| c c c c | c c}
\toprule
$\alpha$ & Early OFR & Late OFR & OFR Drop & Rel. Drop & Mean Weight & Verifier Pass Rate \\
\midrule
0.1  & 0.741 & 0.253 & 0.487 & 65.8\% & $1.07 \rightarrow 1.02$ & $24.8\% \rightarrow 71.4\%$ \\
0.25 & 0.739 & 0.250 & 0.489 & 66.2\% & $1.17 \rightarrow 1.06$ & $25.6\% \rightarrow 71.2\%$ \\
\textbf{0.5}  & \textbf{0.768} & \textbf{0.209} & \textbf{0.560} & \textbf{72.9\%} & $\mathbf{1.35 \rightarrow 1.09}$ & $\mathbf{22.9\% \rightarrow 73.2\%}$ \\
1.0  & 0.732 & 0.239 & 0.423 & 57.8\% & $1.67 \rightarrow 1.29$ & $26.0\% \rightarrow 67.5\%$ \\
2.0  & 0.744 & 0.249 & 0.495 & 66.5\% & $1.68 \rightarrow 1.22$ & $24.9\% \rightarrow 71.7\%$ \\
\bottomrule
\end{tabular}
\end{table}

%==============================================================================
\section{Additional Experiment Results and Analysis}
\label{app:AER}
%==============================================================================

This appendix expands the experiments of \S\ref{sec:exp_single}
(single-turn SDPO\,+\,PH on Science Q\&A and Tool Use) and
\S\ref{sec:exp_multi} (multi-turn RL\,/\,OPD\,/\,Combined + PH) with the
full $\alpha$-sweeps that the main tables abridge, finer-grained
training-time dynamics, and per-turn breakdowns. All configurations
match App.~\ref{app:experimental_setup} unless stated otherwise.

%------------------------------------------------------------------------------
\subsection{Single-Turn $\alpha$-Sweep: SDPO\,+\,PH on Science Q\&A}
\label{app:single_alpha}
%------------------------------------------------------------------------------

\paragraph{PH cell composition is qualitatively robust across $\alpha$.}
Figure~\ref{fig:single_alpha_composition} repeats the per-step
cell-composition plot of Figure~\ref{fig:single_dynamics} (which fixes
$\alpha{=}1.0$) for the additional weights
$\alpha\in\{0.1, 0.25, 0.5, 2.0\}$ on Science Q\&A. The four panels share a
common pattern: the overconfident-failure (OF) cell shrinks throughout
training while the confident-success (CS) cell grows, and the
underconfident-success (US) and aware-failure (AF) cells stay small.
The fact that this dynamics holds essentially unchanged across two orders
of magnitude of $\alpha$ tells us that PH is not a delicate
optimization-noise instrument: it is a structural reweighting whose
direction (compress OF, grow CS) is set by the data, with $\alpha$
controlling only the speed of that compression.

\paragraph{OFR decreases monotonically and the PH weight self-extinguishes.}
Figure~\ref{fig:single_alpha_dynamics} (left) plots OFR over training for
each $\alpha$ on Science Q\&A. All five trajectories descend monotonically;
$\alpha{=}0.5$ reaches one of the lowest late-stage OFR values,
consistent with the main-text observation in
\S\ref{sec:exp_single}. Figure~\ref{fig:single_alpha_dynamics} (right)
plots the corresponding mean PH weight $\E[w_t]$: it starts above $1$
(because $\xi_t > 0$ on many rollouts initially), decays toward $1$ as
$\xi_t\!\to\!0$ on the bulk of rollouts, and the larger the $\alpha$ the
higher the initial weight but also the steeper the relaxation. The
weight curriculum is therefore self-extinguishing --- exactly the
behaviour predicted by the continuous-relaxation analysis of
App.~\ref{app:continuous}.

\paragraph{Why $\alpha{=}0.5$ is the Science Q\&A sweet spot.}
Table~\ref{tab:chemistry_alpha_dynamics} numerically compares the five
$\alpha$ values along three axes: late-training OFR, the early-to-late
OFR drop, and the verifier pass rate at the final checkpoint.
$\alpha{=}0.5$ is best on all three: largest absolute OFR drop
($0.560$), largest relative drop ($72.9\%$), and highest final verifier
pass rate ($73.2\%$). The weight column shows why: $\alpha{=}0.5$'s
mean weight starts at $1.35$ (large enough to materially reweight
miscalibrated samples) and anneals to $1.09$ (small enough to not
perturb late-stage learning), giving the cleanest curriculum.
$\alpha{=}1.0$ and $\alpha{=}2.0$ start the curriculum from much higher
initial weights ($1.67, 1.68$) and do not anneal as cleanly, trading a
fraction of the calibration improvement for slightly noisier capability.
Conversely, $\alpha\in\{0.1, 0.25\}$ are nearly identity reweightings
throughout training (initial weight $1.07, 1.17$) and so leave the OFR
drop close to the small-$\alpha$ floor of $\approx 0.49$.

\begin{figure*}[!h]
  \centering
\includegraphics[width=0.49\linewidth, clip]{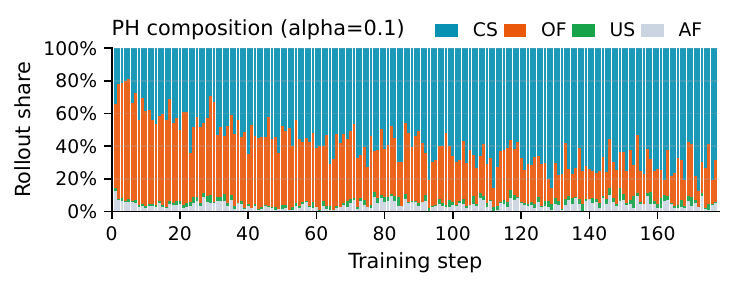}
\includegraphics[width=0.49\linewidth, clip]{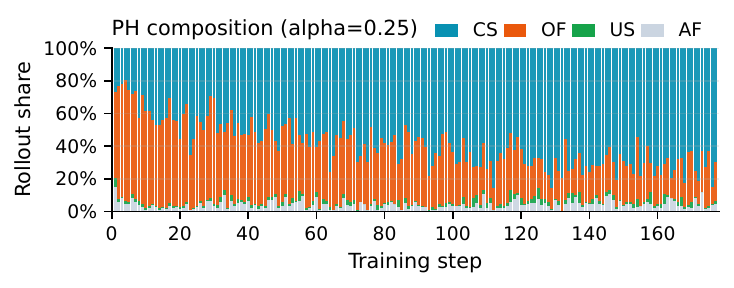}
\includegraphics[width=0.49\linewidth, clip]{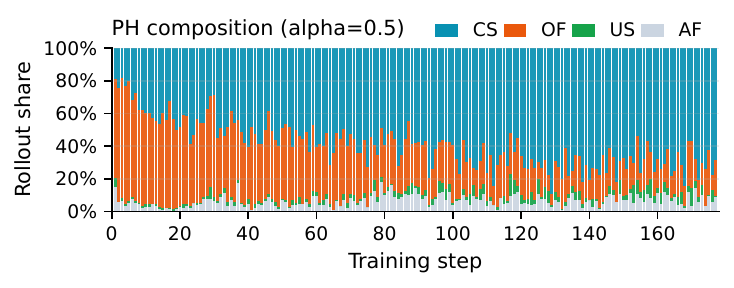}
\includegraphics[width=0.49\linewidth, clip]{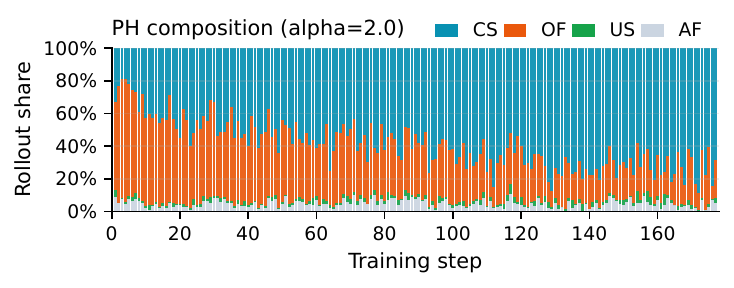}
  \caption{%
    \textbf{PH cell composition across $\alpha$ on Chemistry.}
    Per-step decomposition of training rollouts into CS / OF / US / AF
    for $\alpha\in\{0.1, 0.25, 0.5, 2.0\}$ (the $\alpha{=}1.0$ panel is
    in Figure~\ref{fig:single_dynamics}). The compress-OF / grow-CS
    pattern is qualitatively unchanged across two orders of magnitude
    of $\alpha$.
  }
  \label{fig:single_alpha_composition}
\end{figure*}

\begin{figure*}[!h]
  \centering
\includegraphics[width=0.99\linewidth, clip]{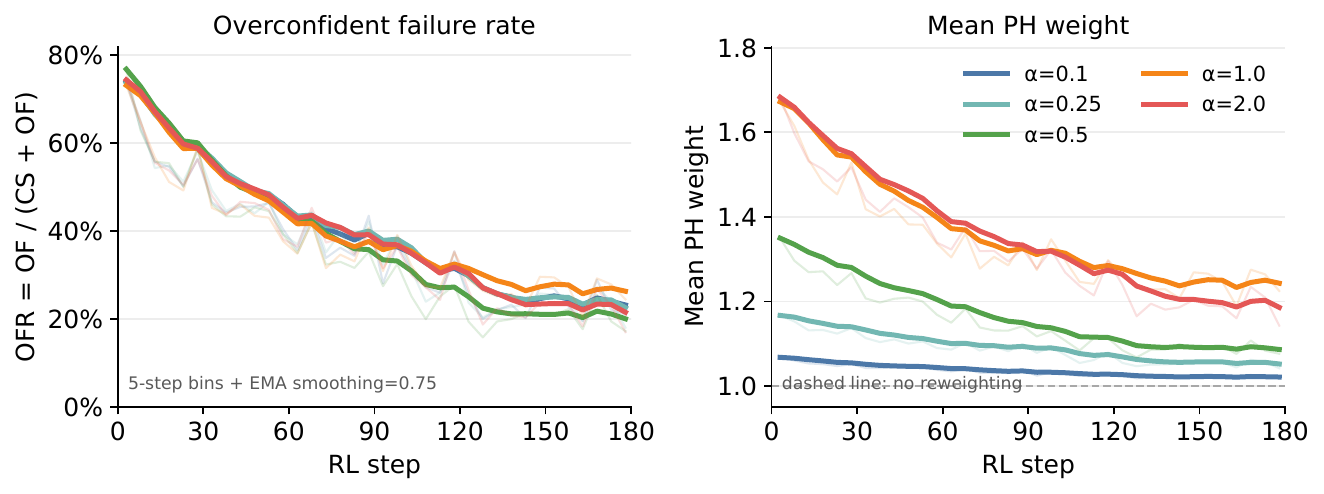}
  \caption{%
    \textbf{OFR and PH-weight dynamics across $\alpha$ on Science Q\&A.}
    Left: OFR over training for each $\alpha$; all five trajectories
    descend monotonically. Right: the mean PH weight starts above $1$
    and anneals toward $1$ as $\xi_t\!\to\!0$ (App.~\ref{app:continuous}).
    Faint lines: 5-step bins; solid lines: EMAs.
  }
  \label{fig:single_alpha_dynamics}
\end{figure*}

%------------------------------------------------------------------------------
\subsection{Multi-Turn $\alpha$-Sweep: Per-Round Capability and Per-Turn Calibration}
\label{app:AER:multi}
%------------------------------------------------------------------------------

\paragraph{$\alpha$ sweep across capability and per-turn calibration.}
Figure~\ref{fig:multi_alpha_sweep} plots the multi-turn $\alpha$-sweep on
the Combined base method under the GPT-4.1 simulator: teacher-side score
across training rounds (left), per-turn surprise rate across conversation
turns (middle), and per-turn OFR across conversation turns (right). The
capability trajectory (left) reproduces the headline of
Table~\ref{tab:multi_results}: $\alpha{=}1.0$ rises fastest and plateaus
highest, $\alpha{=}0.5$ tracks closely, and $\alpha{=}2.0$ exhibits the
highest round-to-round variance. The per-turn surprise (middle) and
per-turn OFR (right) reveal a structural pattern that the
trajectory-averaged metrics in the main text do not surface:
\emph{later turns have higher initial overconfidence}, consistent with
the long-horizon credit-assignment intuition that the agent's self-model
is least developed for late, context-saturated turns; and PH compresses
exactly this region most. At $\alpha{=}1.0$, late-turn OFR drops to
levels comparable to the early-turn OFR of the baseline, effectively
flattening the per-turn calibration profile.

\paragraph{When $\alpha{=}2.0$ overshoots in long horizons.}
Figure~\ref{fig:multi_alpha_sweep} also concretizes the $\alpha{=}2.0$
caveat from \S\ref{sec:exp_multi}: the right panel shows that
$\alpha{=}2.0$ achieves the lowest per-turn OFR overall, but the left
panel shows that its capability oscillates round-to-round and ends below
$\alpha{=}1.0$. The mechanism is the variance-amplification effect
quantified in App.~\ref{app:variance}: very large $\alpha$ multiplies the
per-rollout combined-advantage by up to $1+\alpha=3$ on miscalibrated
rollouts, and in long-horizon training those rollouts also tend to be
the longest token sequences (overconfident failures are correlated with
more aggressive generation), so the per-batch gradient norm spikes
relative to the baseline distribution. We did not observe this regime in
single-turn (Table~\ref{tab:single_results}) because the per-rollout
token count there is bounded and similar across all rollouts; the
multi-turn regime, with its variable-length sessions of up to
$T_{\max}{=}48$ turns (App.~\ref{app:setup_multi}), exposes it.

\paragraph{Practitioner recommendation.}
For both regimes, a single $\alpha{=}1.0$ default works well across
methods, datasets, and simulators. The per-domain optimum may shift to
$\alpha{=}0.5$ (Science Q\&A, Combined) or remain at $\alpha{=}1.0$ (Tool
Use, RL), but no $\alpha$ outside the $\{0.5, 1.0\}$ band ever beats the
in-band defaults on \emph{both} capability and calibration
simultaneously. We therefore recommend $\alpha{=}1.0$ as the
practitioner default and $\alpha{=}0.5$ as a slightly more conservative
alternative whenever capability stability is paramount.

% \section{Additional Experiment Results}
% \label{app:AER}

% SPDO 

\begin{figure*}[!h]
  \centering
\includegraphics[width=0.99\linewidth, clip]{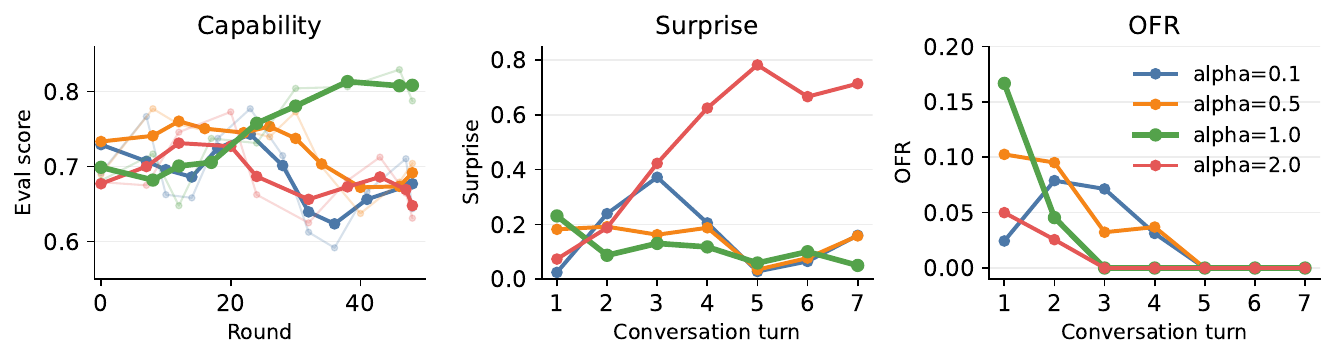}
  \caption{%
    \textbf{Multi-turn $\alpha$-sweep (Combined, GPT-4.1 simulator).}
    Left: teacher-side evaluation score across training rounds. Middle:
    teacher-side surprise rate across conversation turns. Right:
    teacher-side OFR across conversation turns. Curves: $\alpha\in\{0.1,
    0.5, 1.0, 2.0\}$. Later turns start more overconfident than early
    turns, and PH compresses this region most.
  }
  \label{fig:multi_alpha_sweep}
\end{figure*}

\begin{figure*}[!h]
  \centering
\includegraphics[width=0.99\linewidth, clip]{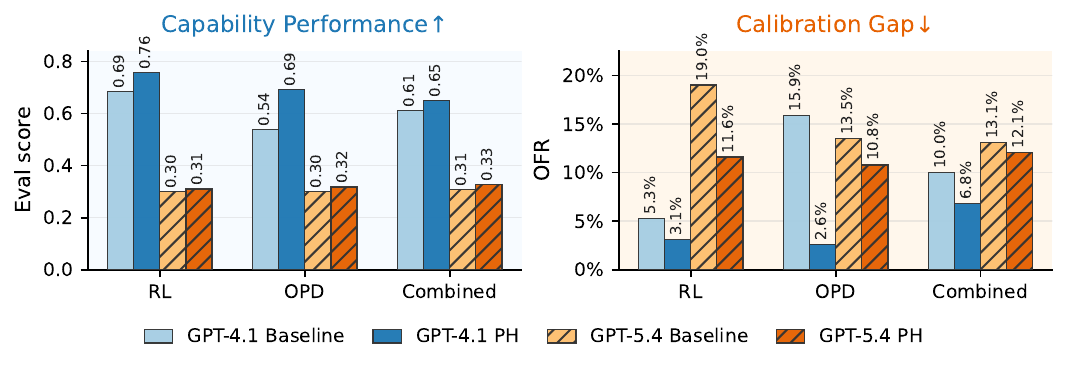}
% \caption{
% Capability and calibration of OLMo3 8k mixed personal-agent training.
% We compare RL, OPD, and Combined training under GPT-4.1 and GPT-5.4 simulators.
% Left: teacher evaluation performance, where higher scores indicate better capability.
% Right: teacher-side overconfidence rate, $\mathrm{OFR}=\mathrm{OF}/(\mathrm{CS}+\mathrm{OF})$, where lower values indicate better calibration.
% Baseline denotes the small-alpha PH setting; for GPT-4.1, this is extracted from $\alpha=0.1$.
% }
%   \label{fig:1}
  \caption{%
    \textbf{Capability and calibration across user simulators.}
    OLMo-3-7B-Instruct trained on the multi-turn personal-agent task
    with $8\text{k}$ response-length budget and mixed Student/Teacher
    rounds, under two simulators (GPT-4.1 and GPT-5.4). Left:
    teacher-side personalization-evaluator score (higher is better).
    Right: teacher-side OFR (lower is better).
  }
  \label{fig:multi_simulator}
\end{figure*}

%==============================================================================
\section{Prompt Templates}
\label{app:prompts}
%==============================================================================

This appendix lists the prompt templates used by PH's prospective self-evaluator (\S\ref{sec:self_evaluator}) and, in multi-turn settings, the retrospective process reward model (PRM) used as the verifier $R$ in Eq.~\ref{eq:reward}. All prompts elicit a ternary score parsed from a \texttt{\textbackslash boxed\{$\cdot$\}} field as $z_t$ or $r_t \in \{-1, 0, +1\}$; ties under majority voting are broken toward $0$.

\subsection{Single-Step Tasks}
\label{app:prompts_singlestep}

For single-step tasks (Science Q\&A, Tool-Use, etc.), the prospective self-evaluator uses a task-type-specific template that wraps a common ternary scoring instruction. The task type is routed automatically from the dataset identifier: math (GSM8K, MATH, AIME, etc.) $\to$ \texttt{math}; multiple-choice (SciKnowEval, GPQA, MMLU) $\to$ \texttt{mcq}; tool-use $\to$ \texttt{tool\_use}; code (LiveCodeBench, HumanEval+, etc.) $\to$ \texttt{code}; otherwise $\to$ \texttt{default}. Placeholders \texttt{\{problem\}} and \texttt{\{solution\}} are filled with the task instance and the agent's just-emitted action $a_t$.

The common ternary scoring instruction below is appended verbatim to every single-step template (Listing~\ref{lst:prompt_scoring}); each task type then wraps it with its own header (Listings~\ref{lst:prompt_math}--\ref{lst:prompt_default}).

\begin{lstlisting}[style=promptstyle, caption={Common ternary scoring instruction (single-step).}, label={lst:prompt_scoring}]
Return exactly one score and no explanation:
+1 if you are confident it is correct
-1 if you believe it is wrong
0 if you are unsure
Your entire response must be exactly one of: +1, 0, -1
\end{lstlisting}

\begin{lstlisting}[style=promptstyle, caption={Single-step prospective self-evaluator: Math template.}, label={lst:prompt_math}]
Review your solution to the following problem.

Problem:
{problem}

Your solution:
{solution}

Do you believe your final answer is correct?
[common scoring instruction]
\end{lstlisting}

\begin{lstlisting}[style=promptstyle, caption={Single-step prospective self-evaluator: Multiple-choice template.}, label={lst:prompt_mcq}]
Review the question and your chosen answer below.

Question:
{problem}

Your response:
{solution}

Do you believe you selected the correct answer?
[common scoring instruction]
\end{lstlisting}

\begin{lstlisting}[style=promptstyle, caption={Single-step prospective self-evaluator: Tool-use template.}, label={lst:prompt_tool}]
Review the task and your proposed solution below.

Task:
{problem}

Your solution:
{solution}

Do you believe you selected the correct actions and inputs to
complete this task?
[common scoring instruction]
\end{lstlisting}

\begin{lstlisting}[style=promptstyle, caption={Single-step prospective self-evaluator: Code template.}, label={lst:prompt_code}]
Review the problem and your code solution below.

Problem:
{problem}

Your solution:
{solution}

Do you believe your code produces the correct output for all
test cases?
[common scoring instruction]
\end{lstlisting}

\begin{lstlisting}[style=promptstyle, caption={Single-step prospective self-evaluator: Default template (fallback).}, label={lst:prompt_default}]
Review your solution to the following problem.

Problem:
{problem}

Your solution:
{solution}

Do you believe your solution is correct?
[common scoring instruction]
\end{lstlisting}

\subsection{Multi-Turn Tasks}
\label{app:prompts_multiturn}

In multi-turn settings, two distinct prompts are used: a \emph{prospective} prompt that the agent's policy answers immediately after emitting $a_t$ (without observing the next state $c_t = s_{t+1}$), and a \emph{retrospective} PRM prompt that scores the same response after $s_{t+1}$ is revealed. The retrospective PRM serves as the verifier $R$ in Eq.~\ref{eq:reward}; the prospective prompt produces $z_t$ as in Eq.~\ref{eq:eval}. Placeholders are filled at runtime: \texttt{\{response\_text\}} is the agent's action $a_t$, \texttt{\{last\_user\_message\}} is the most recent user turn from the conversation history, \texttt{\{next\_state\_text\}} is $s_{t+1}$, and \texttt{\{next\_state\_role\}} indicates whether the next state is a user reply or a tool return.

\paragraph{Prospective self-evaluator (general multi-turn).}
This is the default prospective prompt used by PH at every turn $t$. It receives only $(s_t, a_t)$ via \texttt{\{last\_user\_message\}} and \texttt{\{response\_text\}}; the next state $s_{t+1}$ is not exposed.

\begin{lstlisting}[style=promptstyle, caption={Multi-turn prospective self-evaluator (general).}, label={lst:prompt_multi_general}]
User request:
{last_user_message}

Assistant response:
{response_text}

Does this response correctly and completely address the user's
request?
Answer \boxed{1} for yes, \boxed{-1} for no, \boxed{0} for unsure.
Be brief. End your answer with \boxed{SCORE}.
\end{lstlisting}

\paragraph{Prospective self-evaluator (Personal Agent, style-aware).}
This variant is used in the Personal Agent setting where user satisfaction depends on stylistic conformance (natural prose, no AI-style formatting) in addition to factual correctness. The prompt is split into a system message that defines the user's style preferences and a user message that asks for the prediction.

\begin{lstlisting}[style=promptstyle, caption={Multi-turn prospective self-evaluator (Personal Agent, style-aware).}, label={lst:prompt_multi_personal}]
[SYSTEM]
You are predicting whether a user will be SATISFIED or
DISSATISFIED with an AI assistant's response, based on the
user's strong style preferences.

The user HATES AI-style formatting. Specifically, the user
will ask for a rewrite if the response contains ANY of these:
- Markdown bold (**text**), headers (## or ###), or horizontal
  rules
- Numbered lists (1. 2. 3.) or bullet points (- or *)
- Overly structured step-by-step layouts
- Phrases like "Here is", "Sure!", "Great question!", "Let me",
  "Certainly"
- Emojis or excessive exclamation marks

The user WANTS natural, human-like prose, as if a real person
(student or teacher) wrote it casually. Flowing sentences, no
formatting artifacts.

Scoring:
- \boxed{-1}: response has formatting issues or sounds AI-
  generated; user will request a rewrite
- \boxed{1}: response is natural, human-like prose with no
  formatting artifacts; user will be satisfied
- \boxed{0}: borderline / unsure

Be strict. Even one instance of bold text or a numbered list
means \boxed{-1}.

[USER]
User message:
{last_user_message}

Assistant response:
{response_text}

Check the response for AI-style formatting (bold, headers,
numbered lists, bullet points, AI phrases). Does it read like
natural human writing?
Answer with \boxed{SCORE}.
\end{lstlisting}

\paragraph{Retrospective PRM (verifier).}
The PRM observes the next state $s_{t+1}$ as evidence and assigns the outcome $r_t \in \{-1, 0, +1\}$.

\begin{lstlisting}[style=promptstyle, caption={Multi-turn retrospective PRM verifier (next-state grounded).}, label={lst:prompt_multi_prm}]
[SYSTEM]
You are a process reward model (PRM) evaluating an AI
assistant. You will see the assistant's output and the
subsequent next state. Your task: decide whether the
assistant's output successfully fulfilled the user's intent
at that step, using the next state as evidence.

Understanding the next state's role:
- role='user': a reply from the user.
- role='tool': the return value of a tool the assistant
  invoked. This content was NOT available before the
  assistant's action; it exists BECAUSE the assistant called
  the tool. A successful, non-error tool output means the
  assistant's action worked correctly and should be scored
  positively.

Scoring rules:
- \boxed{1}  (good):    next state shows the task progressed
                        as expected (user moves on, says
                        thanks, environment confirms success,
                        tool returns successful non-error
                        result).
- \boxed{-1} (bad):     next state signals the assistant's
                        output was wrong, incomplete, or
                        unwanted. Key negative signals:
                        * user asks the assistant to redo,
                          retry, or repeat the same action
                        * user requests a correction or
                          modification
                        * user rephrases or restates the same
                          request
                        * environment returns an error,
                          failure, or unexpected result.
- \boxed{0}  (neutral): next state is ambiguous.

A change request is negative feedback: it means the previous
output did not meet the user's need; do not treat it as a
neutral new instruction. Think step-by-step, then give your
final score inside \boxed{}.

[USER]
Assistant output:
{response_text}

Next state [role: {next_state_role}]:
{next_state_text}

First, classify the next state: is it (a) positive
progression, (b) a correction / redo / change request, or
(c) ambiguous? Then assign \boxed{1}, \boxed{-1}, or \boxed{0}.
\end{lstlisting}

% \subsection{Implementation Notes}
% \label{app:prompts_impl}

% \paragraph{Score parsing.} For all templates, the parser first attempts to match an exact one-token response (\texttt{+1}, \texttt{0}, \texttt{-1}), then a strict \texttt{\textbackslash boxed\{$\cdot$\}} field, then a relaxed \texttt{boxed\{$\cdot$\}} variant, then a labeled token (\texttt{score: $\cdot$}, \texttt{answer: $\cdot$}), and finally falls back to the last standalone score-like token in the trailing 200 characters of the response. If none of these match, the sample is marked invalid (\texttt{pred\_valid = False}) and excluded from the PH weight computation.

% \paragraph{Voting and tie-breaking.} For $M > 1$ samples, a strict majority (single most-frequent score) is used; under a tie or when all samples are invalid, the aggregated score defaults to $0$ (honest uncertainty), preserving a neutral weight.

% \paragraph{Generation parameters.} The prospective predictor uses temperature $0.6$ and a max generation budget of $256$ tokens by default; the retrospective PRM uses temperature $0.6$ and up to $4096$ tokens to accommodate chain-of-thought before the final \texttt{\textbackslash boxed\{$\cdot$\}}.

\newpage
\clearpage
\section{Limitations and Future Work}
\label{sec:limitations}

While Prospective Hindsight (PH) provides a robust and simple mechanism for self-calibration in agent reinforcement learning, we identify several limitations that offer avenues for future work:

\begin{itemize}[leftmargin=1.6em]
    \item \textbf{Variance-Amplification at High $\alpha$:} High values of the surprise amplification hyperparameter ($\alpha \ge 2.0$) can lead to gradient norm spikes, particularly in multi-turn, long-horizon tasks where overconfident failures are correlated with longer token sequences . While $\alpha=1.0$ is a stable default, extreme upweighting may destabilize training.
    
    \item \textbf{Theoretical Assumptions:} Our gradient-flow analysis, which shows the PH update is strictly calibration-aligned, relies on a ``uniform-base-loss'' assumption—specifically, that the base loss is approximately constant on surprising rollouts. In practice, this alignment is quantitative rather than absolute, as the per-rollout loss varies with specific log-probabilities.
    
    \item \textbf{Training-Time Compute Overhead:} PH requires $M$ additional forward passes per rollout step to elicit the prospective belief. While this adds negligible wall-clock time in asynchronous multi-turn settings (approx. 0\%), it introduces a small overhead (approx. 6\%) in single-turn synchronous training. 
    
    \item \textbf{Dependence on Self-Model Quality:} The effectiveness of the surprise signal depends on the policy's ability to produce meaningful ternary predictions. Although PH facilitates the co-evolution of the self-model and policy, the initial training phase may provide noisier signals if the base model is extremely poorly calibrated.
    
    \item \textbf{Scope of Empirical Evaluation:} Our experiments focused on verifiable tasks (Science Q\&A, Tool Use) and a math-tutoring personal agent. While we demonstrated cross-scale transfer from 7B to 20B models, the performance of PH in purely creative or subjective generation tasks where retrospective rewards are inherently ambiguous remains to be fully explored.
\end{itemize}

The limitations above suggest several directions where the Prospective
Hindsight principle can be extended or its current evidence base
strengthened.

% \paragraph{Direct comparison with calibration-augmented RL.}
% The most immediate empirical extension is a head-to-head benchmark
% against CaOPD, Leng et al.~\cite{leng2024taming}, and Damani et
% al.~\cite{damani2025beyond} on shared single-turn splits, with matched
% seeds and significance testing. Our single-turn pipeline already shares
% the CaOPD preprocessing, so this comparison is methodologically
% straightforward; we expect it would either substantiate the
% capability--calibration trade-off avoidance claim or sharpen our
% understanding of when prior calibration-augmented recipes incur the
% trade-off.
\begin{itemize}[leftmargin=1.6em]

\item \textbf{Independent characterization of the prospective evaluator.}
Decoupling the evaluator's quality from the training trajectory is
important both for methodological cleanliness and for understanding
when PH should be expected to work. Promising directions include:
(i) measuring evaluator accuracy on held-out prompts at fixed
checkpoints, (ii) comparing evaluator outputs against a stronger
external judge, (iii) calibrating the evaluator's token-level
confidence against its boxed score, and (iv) studying the regime in
which evaluator and policy share systematic blind spots and how PH
behaves under such failure.

\item \textbf{Scaling to longer-horizon agentic tasks.}
Our multi-turn experiments cap conversations at $T_{\max} = 48$ turns.
Genuinely long-horizon agents---software engineering trajectories
spanning hundreds of tool calls, multi-stage research assistants, or
extended pair-programming sessions---raise two questions PH does not
yet answer: how does the surprise signal interact with sparse, delayed
verifier feedback at long horizons, and does the per-turn calibration
cell composition (Fig.~\ref{fig:cell_composition}) shift further as
horizons grow? Combining PH with hierarchical credit-assignment
schemes that already address sparse-reward long-horizon RL is a
natural follow-up.

\item \textbf{Continuous-valued PRM scores and richer surprise signals.}
The current surprise indicator is binary (discrete gate) or a
half-distance in $[0, 1]$ (continuous gate). Recent process reward
models emit graded scores, confidence intervals, or per-criterion
breakdowns~\cite{wang2026openclaw}.  Extending PH to
exploit these richer retrospective signals---for example, by defining
$\xi_t$ as a function of the PRM's confidence interval on $r_t$ rather
than its point estimate---may further reduce the variance amplification
at large $\alpha$ that we observe in long horizons.

\item \textbf{Composition with inference-time self-correction.}
PH operates entirely at training time and leaves the agent's deployment
behavior unchanged. Inference-time self-correction methods such as
Reflexion~\cite{shinn2023reflexion-auq} and
Self-Refine~\cite{madaan2023self-auq} act on the policy distribution
\emph{after} PH has shaped it. Whether the two mechanisms compose
constructively---PH producing a better-calibrated self-model, then
inference-time self-correction exploiting that calibration to drive
re-tries---is an open question with practical implications for
deployed agentic systems.

\end{itemize}

% \paragraph{Beyond binary verifiers and rule-based supervision.}
% The empirical evaluations in this paper use either rule-based binary
% verifiers (single-turn) or LLM-based PRM judges (multi-turn). PH's
% surprise signal is well-defined for any retrospective evaluator that
% emits an outcome label, but its behavior under non-verifiable tasks
% (creative writing, open-ended analysis, dialogue without a clear
% success criterion) is unstudied. The conceptual question is whether
% the prediction--reality gap is still informative when "reality" itself
% is reward-model-defined and potentially miscalibrated, and the
% practical question is whether PH amplifies or corrects bias in such
% reward models.

\newpage
\clearpage
\section{Broader Impact}
\label{sec:broader_impact}

The development of Prospective Hindsight (PH) aims to bridge the gap between capability and calibration in language model agents. As LLMs are increasingly deployed in autonomous agentic roles—such as personal assistants, tool-use interfaces, and tutoring systems—the societal implications of their reliability become paramount.

\textbf{Positive Societal Impacts:}
\begin{itemize}[leftmargin=1.6em]
    \item \textbf{Enhanced Safety and Reliability:} By explicitly penalizing overconfident failures, PH encourages agents to recognize their own limitations. This reduction in "confidently wrong" behavior is critical in high-stakes domains where an agent should ideally invite human intervention rather than proceeding with a flawed trajectory.
    \item \textbf{Improved Human-AI Trust:} Calibration is a cornerstone of effective human-AI collaboration. Agents that can accurately communicate their uncertainty (through honest uncertainty markers) allow users to set appropriate expectations and intervene when necessary.
    \item \textbf{Resource Efficiency:} PH is a self-extinguishing curriculum that reduces the need for manually designed reward schedules or expensive, separate calibration-specific training phases. This can lower the computational and human-labor barriers to developing reliable AI systems.
\end{itemize}

\textbf{Potential Negative Impacts and Mitigations:}
\begin{itemize}[leftmargin=1.6em]
    \item \textbf{Risk of Deceptive Calibration:} While PH aligns the agent's prospective belief with verifier outcomes, there is a theoretical risk that an agent could learn to "appear" calibrated by avoiding difficult but necessary tasks to minimize surprise. We mitigate this by using PH as a plug-in on top of task-success objectives, ensuring capability is not sacrificed for calibration.
    \item \textbf{Over-reliance on Automated Feedback:} The method relies on the quality of retrospective verifiers or PRMs. If these evaluators contain biases or incorrect logic, PH may amplify these errors by aggressively upweighting rollouts that "surprise" the model based on flawed ground truth. Developers should ensure that the underlying verifiers are robust and subject to human auditing.
    \item \textbf{Dual-Use Concerns:} Improved agentic autonomy and tool-use capabilities, while beneficial for productivity, could be misused to automate malicious activities. However, PH’s primary contribution is the \textit{alignment} of internal belief with reality, which is a fundamentally defensive property aimed at reducing unintended failures rather than increasing raw harmful capabilities.
\end{itemize}

\end{document}